\RequirePackage[l2tabu,orthodox]{nag}     
\documentclass[a4paper,11pt]{llncs}
\usepackage[a4paper,margin=2.5cm]{geometry}

\usepackage[utf8]{inputenc}
\usepackage[T1]{fontenc}
\usepackage{mathptmx}
\usepackage{zi4}

\usepackage{cite}
\usepackage{mathtools,amsfonts}
\usepackage{url}
\usepackage[british]{babel}
\usepackage[final,activate=true,kerning=true,protrusion=true]{microtype}
\usepackage{enumerate}

\usepackage{varioref}
\usepackage{xcolor}
\usepackage{algorithm2e}
\usepackage{booktabs}
\usepackage{todonotes}
\usepackage{csquotes}

\usepackage[framemethod=tikz]{mdframed}

\newcounter{algorithm}

\newcommand{\newalgthin}[2]{%
\begin{mdframed}[linewidth=0.3pt]
\centering\small
\underline{\refstepcounter{algorithm}\textbf{Algorithm \thealgorithm}: #1}
#2
\end{mdframed}}

\newcommand{\newalg}[4]{\newalgthin{#1}{%
\begin{flushleft}
\textbf{Input:} #2\\
\textbf{Output:} #3
\end{flushleft}
#4}}

\newif\iflongversion

\usepackage[pdfpagelabels=true]{hyperref}
\hypersetup{
	linktoc = page,
	pdfpagemode = UseNone,
	colorlinks,
	linkcolor={red!50!black},
	citecolor={blue!50!black},
	urlcolor={blue!80!black}
}

\usepackage{paralist}
\usepackage{ragged2e}
\usepackage{tikz}
\usepackage{tikz-qtree}
\usetikzlibrary{arrows}
\usetikzlibrary{backgrounds, decorations.pathreplacing, bending}

\newcommand{\bl}[1]{\textcolor{blue}{\texttt{#1}}}

\newcommand{\covide}{\textsc{covid-19}}
\newcommand{\covid}{\textsc{covid-19~}}

\begin{document}

\title{\textsc{optimal \covid pool testing with a priori information}}

\longversiontrue

\author{
	Marc~Beunardeau\thanks{Work done when the author was with ÉNS and Ingenico Laboratories.}  \and 
	Éric~Brier\inst{2} \and 
    No\'emie~Cartier\inst{1} \and 	
	Aisling~Connolly\inst{1,2} \and 
    Nathana\"el~Courant\inst{1} \and 
	Rémi~Géraud-Stewart\inst{1,2} \and 
	David~Naccache\inst{1,3}\and	
	Ofer~Yifrach-Stav\inst{1} 
}
\institute{
ÉNS (DI), Information Security Group,
CNRS, PSL Research University, 75005, Paris, France.\\
	\email{\url{given\_name.family\_name@ens.fr}}\\
	\and
	Ingenico Laboratories, 75015, Paris, France.\\
	\email{\url{given\_name.family\_name@ingenico.com}}\\
	\and 
	School of Cyber Engineering, Xidian University, Xi'an, 710071, PR China\\
		\email{\url{david@xidian.edu.cn}}
}

\maketitle

\begin{abstract} 
As humanity struggles to contain the global \covid infection, prophylactic actions are grandly slowed down by the shortage of testing kits.

Governments have taken several measures to work around this shortage: the FDA has become more liberal on the approval of \covid tests in the US. In the UK emergency measures allowed to increase the daily number of locally produced test kits to 100,000. China has recently launched a massive test manufacturing program.

However, all those efforts are very insufficient and many poor countries are still under threat.

A popular method for reducing the number of tests consists in \emph{pooling samples}, i.e. mixing patient samples and testing the mixed samples once. If all the samples are negative, pooling succeeds at a unitary cost. However, if a single sample is positive, failure does not indicate \emph{which} patient is infected.

This paper describes how to optimally detect infected patients in pools, i.e. using a minimal number of tests to precisely identify them, given the a priori probabilities that each of the patients is healthy. Those probabilities can be estimated using questionnaires, supervised machine learning or clinical examinations.

The resulting algorithms, which can be interpreted as informed divide-and-conquer strategies, are non-intuitive and quite surprising. They are patent-free. Co-authors are listed in alphabetical order.
\end{abstract}

\section{Introduction and motivation}

\subsection{Testing for \covid infection}

The current \covid (Coronavirus Disease 2019) pandemic is rapidly spreading and significantly impacts healthcare systems. Stay-at-home and social distancing orders enforced in many countries are supporting the control of the disease's spread, while causing turmoil in the economic balance and in social structures\cite{baker2020unprecedented}. 
Rapid detection of cases and contacts is an essential component in controlling the spread of the pandemic. In the US, the current estimation is that at least 500,000 Covid-19 tests will need to be performed daily in order to successfully reopen the economy \cite{thecrimson:2020}
\smallskip

Unfortunately, as humanity attempts to limit the global \covid infection, prophylactic actions are grandly slowed-down by the severe shortages of \covid testing kits \cite{ellis:2020}.\smallskip

There are currently two types of tests for \covide: 
\begin{itemize}
    \item 
\emph{Molecular diagnostic tests} that detect the presence of  \textsc{sars-cov-2} nucleic acids in human samples. A positive result of these tests indicates the presence of the virus in the body. 
\item 
\emph{Serological diagnostic tests} that identify antibodies (e.g., IgM, IgG) to \textsc{sars-cov-2} in clinical specimens \cite{wang2020combination}. Serological tests, also known as \emph{antibody tests}, can be helpful in identifying not only those who are ill, but also those who have been infected, as antibodies are still present in their blood. This identification may be important for several reasons. First, this test can differentiate those who are immune to the virus and those who are still at risk. Secondly, identifying populations who have antibodies can facilitate research on the use of convalescent plasma in the development of a cure for \covid \cite{FDA:2020}.
\end{itemize}
As mentioned, both tests are in very short supply.
Governments have taken several measures to work around this shortage: the FDA\footnote{United States Food and Drug Administration} has become more liberal on the approval of \covid tests via the EUA (Emergency Use Authorization) \cite{FDA:2020}; in the US and the UK there are attempts to boost the number of locally produced test kits to reach a throughput of 100,000 kits per day. Those efforts cannot, however, be followed by many countries and there remains entire swaths of Africa, Asia and Latin America that are under concrete threat. 

\subsection{Pool testing}

To optimize the use of available tests, reduce costs and save time, \emph{pool testing} (or \emph{group testing}) can be considered. The concept is credited to Dorfman \cite{Dorfman} who wished to test US servicemen for syphilis. In pool testing, multiple samples are mixed, and the resulting \enquote{pool} is tested using the same amount of material or equipment that would have been required to test one individual sample. \cite{111,222} are important refinements of Durfman's work.

However, when at least one sample in the pool is positive, then the pool test fails. This means that (at least) one sample in the pool is positive, but gives no information about which one. The most naive approach is then to re-test individually each sample, which can be costly and time consuming.

The area has seen numerous developments. \cite{DBLP:journals/corr/abs-1902-06002} provides a recent survey of the topic and \cite{book} is the reference book on the topic.\smallskip

It is important to distinguish between two types of pool tests: \textsl{Adaptive tests} where the tested samples depend on previously tested ones and \textsl{Nonadaptive tests}, where all the tests are planned in advance. 

Pool tests are also classified as either probabilistic or combinatorial. In essence, probabilistic models assume a probability distribution and seek to optimize the average number of tests required to test all the patients. By opposition, combinatorial algorithms seek to minimize the worst-case number of tests when the probability distribution governing the infection is unknown.

This paper deals with adaptive probabilistic tests.

\subsection{Related Research}

Pool testing has already been used to screen large portions of the population in scarcely-infected areas (or as a best-effort measure, when test availability was low). 
Pool testing has been successfully used to identify viral diseases, such as \textsc{hiv} \cite{nguyen2019methodology}, \textsc{zika} \cite{bierlaire2017zika}, and \textsc{Influenza} \cite{Van891}. In addition, Pool testing has been suggested as a screening method for routine \textsc{hcv}, \textsc{hbv}, and \textsc{hiv} -1 PCR donors for a blood-bank \cite{roth1999feasibility}.
In light of the recent pandemic and the urging need to test vast number of subjects, the idea of Pool testing is becoming more and more appealing.  \smallskip

It is currently the official testing procedure in Israel \cite{BenAmi2020}, Germany, South Korea\cite{Korea:2020}, and some US\footnote{e.g Nebraska.} \cite{Nebraska:2020} and Indian\footnote{e.g. Uttar Pradesh, West Bengal, Punjab, Chhattisgarh, Maharashtra.}  states\footnote{Indian Council of Medical Research, \emph{Advisory on feasibility of using pooled samples for molecular testing of \covid}, April 13, 2020.} \cite{India:2020}. Field research focusing on reducing the number of tests \cite{Farfan2020,assad2020,gollier2020group} did not analyse prior information strategies but instead provided simulation (or small sample) results showing the benefits of pool testing. In most cases, the existing literature only uses pooling as a way to screen the infection in an emerging context, not as a precise approach to identify which individuals are infected and which are not. \smallskip

We also note projects meant to reduce the amount of work required for pool testing: e.g. the Origami Assays \cite{origami} Project, a collection of open source pool testing designs for standard 96 well plates. The Origami XL3 design tests 1120 patients in 94 assay wells.


Yelin et al. \cite{yelin2020evaluation} demonstrated that pool testing can be used effectively to identify one positive \textsc{sars-cov-2} result within 32 samples, and possibly within 64 samples if the cycles are amplified, with an estimated false-negative rate of 10\%. \cite{Taufer2020.04.06.028431} uses a strategy consisting in running \enquote{cross batches}, where the same individuals are tested several times but in different pools, which eventually leads to positive sample identification. The resulting approach ends up using more tests overall (since it tests every individual more than once) than the strategy proposed in this work and does not exploit prior information. 
Similarly, Sinnott-Armstrong et al. \cite{sinnott2020evaluation} suggested to identify low-risk individuals (i.e. asymptomatic and mild cases) and to test them as a pool using a matrix-based method, so as to reduce the number of tests required by up to eight-fold, depending on the prevalence. 

It is hence plausible to assume that a successful emergency application of the refined pool testing procedures described in this paper would improve the \covid testing capacity significantly. 

\smallskip



\subsection{Our Contribution}

This paper departs from the above approaches by assuming the availability of extra information: the a priori probability that each given test is negative. In practice, we may either assume that such probabilities are given, estimated from patient trust metrics, or are learned from past \covid tests. We assume in this work that these probabilities are known.

We show that it is possible to find positive samples in an optimal way, i.e., by performing on average the minimum number of tests. This turns out to be faster than blind divide-and-conquer testing in the vast majority of settings.

A concrete consequence of this research is the design of testing procedures that are faster and more cost-effective.

\section{Intuition}

Before introducing models and general formulae, let us provide the intuition behind our algorithms.

Let us begin by considering the very small case of two samples. These can be tested individually or together, in a pool. Individual \covid testing claims a minimum two units of work---check one sample, then check the other. Pool-checking them requires a minimum of one \covid test. If it is highly probable that both samples are negative, then pool testing is interesting: If both samples are indeed negative, we can make a conclusion after one test and halve the \covid test's cost. However, if that fails, we are nearly back to square one: One of these samples (at least) is positive, and we don't know which one.

In this paper, we identify \emph{when} to check samples individually, and when to pool-check them instead---including all possible generalizations when there are more than 2 samples. We assume that the probability of a sample being positive is known to us in advance. The result is a testing \enquote{metaprocedure} that offers \emph{the best alternative to sequential and individual testing.}

To demonstrate: the testing procedure that always works is to test every sample individually, one after the other: This gives the \enquote{naive procedure}, which always performs $2$ \covid tests, as illustrated in Figure \ref{fig:naive2}. In this representation, the numbers in parentheses indicate which samples are being tested at any given point. The leaves indicate which samples are negative (denoted \textcolor{blue}{\texttt{1}}) or positive (denoted \textcolor{blue}{\texttt{0}}), for instance the leaf \textcolor{blue}{$\texttt{01}$} indicates that only the second sample is healthy. Note that the order in which each element is tested does not matter: There are thus $2$ equivalent naive procedures, namely the one represented in Figure \ref{fig:naive2}, and the procedure obtained by switching the testing order of $(1)$ and $(2)$.
\begin{figure}[!ht]
\centering 
\begin{tikzpicture}[level/.style={sibling distance=2cm, level distance=0.5cm}]
\Tree[
.\node{(1)} ;
	\edge ; [.\node{(2)} ; 
		\edge ; [.\node[blue]{\texttt{11}} ; ] 
		\edge ; [.\node[blue]{\texttt{10}} ; ] 
	]
	\edge ; [.\node{(2)}; 
		\edge ; [.\node[blue]{\texttt{01}} ; ]
		\edge ; [.\node[blue]{\texttt{00}} ; ]
	]
]
\end{tikzpicture}
\caption{The ``naive procedure'' for $n = 2$ consists in testing each entity separately and sequentially.}\label{fig:naive2}
\end{figure}
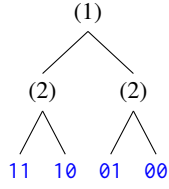

Alternatively, we can leverage the possibility to test both samples together as the set 
$\{1, 2\}$. In this case, pooling the pair $\{1,2\}$ must be the first step: Indeed, testing $\{1, 2\}$ after any other test would be redundant, and the definition of testing procedures prevents this from
happening. If the test on $\{1, 2\}$ is negative, both samples are negative and the procedure immediately yields the outcome \textcolor{blue}{$\texttt{11}$}. Otherwise, we must identify which of the samples 1 or 2 (or both) is responsible for the test's positiveness. There are thus two possible procedures, illustrated in Figure \ref{facto2}.

\begin{figure}[!ht]
\centering
\scalebox{0.8}{
\begin{tikzpicture}[level/.style={sibling distance=2cm, level distance=0.5cm}]
\Tree[
.\node{(1,2)} ;
	\edge ; [.\node[blue]{\texttt{11}} ; ] 
	\edge ; [.\node{(1)} ;
       	\edge ; [.\node[blue]{\texttt{10}} ; ]
        \edge ; [.\node{(2)} ;
           	\edge ; [.\node[blue]{\texttt{01}} ; ]
            \edge ; [.\node[blue]{\texttt{00}} ; ]
        ]
	]
]
\end{tikzpicture}}%
\scalebox{0.8}{
\begin{tikzpicture}[level/.style={sibling distance=2cm, level distance=0.5cm}]
\Tree[
.\node{(1,2)} ;
	\edge ; [.\node[blue]{\texttt{11}} ; ] 
	\edge ; [.\node{(2)} ;
       	\edge ; [.\node[blue]{\texttt{01}} ; ]
        \edge ; [.\node{(1)} ;
           	\edge ; [.\node[blue]{\texttt{10}} ; ]
            \edge ; [.\node[blue]{\texttt{00}} ; ]
        ]
	]
]
\end{tikzpicture}}
\caption{Two pooling testing procedures having $(1, 2)$ as root.}\label{facto2}
\end{figure}
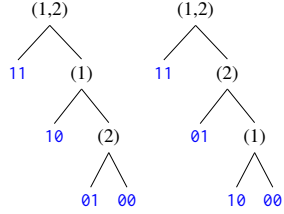

Intuitively, the possibility that this procedure terminates early indicates that, in some situations at least, only one test is performed, and is thus less costly than the naive procedure. However, in some situations up to three tests can be performed, in which case it is more costly than the naive procedure.

Concretely, we can compute how many \covid tests are performed on average by each approach, depending on the
probability $x_1$ that the first sample is positive, and $x_2$ that the second is positive. To each procedure, $\textsf{naive}, \textsf{pool-\textit{left}}, \textsf{pool-\textit{right}}$, we associate the following polynomials representing the expected stopping time:

\begin{itemize}
\item $L_\textsf{{naive}} = 2$
\item $L_\textsf{{pool-\textit{left}}} = (1-x_1)(1-x_2) + 2(1-x_1)x_2 + 3x_1(1-x_2) + 3 x_1x_2$
\item $L_\textsf{{pool-\textit{right}}} = (1-x_1)(1-x_2) + 3(1-x_1)x_2 + 2x_1(1-x_2) + 3 x_1x_2$
\end{itemize}
It is possible to see analytically which of these polynomials evaluates to the smallest value as a function of $(x_1, x_2)$. Looking at Figure \ref{fig:zone2}, we use these expectations to define zones in $[0, 1]^2$ where each algorithm is optimal (i.e. the fastest on average). More precisely, the frontier between zones $C$ and $B$ has equation $x_1 = x_2$, the frontier between $A$ and $B$ has equation $x_2 = (x_1 - 1)/(x_1 - 2)$, the frontier between $A$ and $C$ has equation $x_2 = (2x_1 - 1)/(x_1 - 1)$, and the three zones meet at $\overline x_1 = \overline x_2 = (3 - \sqrt5)/2$, a well-known cutoff value observed as early as 1960 \cite{ungar}.
\begin{figure}[!ht]
\centering 
\scalebox{0.8}{
\begin{tikzpicture}[scale=4]
  \draw (0.5cm,3pt) node[above] {$B$};
  \draw (4pt,0.4cm) node[above] {$C$};
  \draw (0.6cm,0.6cm) node[above] {$A$};
  \draw[->] (-0.05,0) -- (1.1,0) node[right] {$x_1$};
  \draw[->] (0,-0.05) -- (0,1.1) node[above] {$x_2$};
  \draw [scale=1,smooth, gray,dashed] (1,0) --(1,1);
  \draw [scale=1,smooth, gray,dashed] (0,1) --(1,1);

  \foreach \x/\xtext in {0/0, 1/1}
    \draw[shift={(\x,0)}] (0pt,1pt) -- (0pt,-1pt) node[below] {$\xtext$};
  \foreach \y/\ytext in {0/0, 1/1}
    \draw[shift={(0,\y)}] (1pt,0pt) -- (-1pt,0pt) node[left] {$\ytext$};
    
  \draw[scale=1,domain=0:0.382,smooth,variable=\y,red]  plot ({\y},{\y});
  \draw[scale=1,domain=0:0.382,smooth,variable=\x,red]  plot ({\x}, {(2*\x-1)/(\x-1)});
  \draw[scale=1,domain=0.382:1,smooth,variable=\x,red]  plot ({\x}, {(\x-1)/(\x-2)});
\end{tikzpicture}}
\caption{Optimality zones for $n=2$. $A$ : naive procedure; $B$ : pooling procedure (right); $C$ : pooling procedure (left).}\label{fig:zone2}
\end{figure}
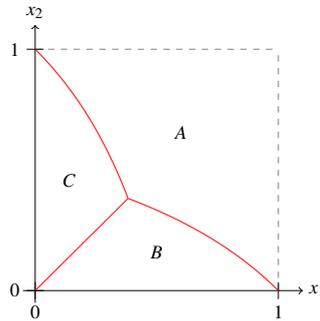

Having identified the zones, we can write an algorithm which, given $x_1$ and $x_2$, identifies in which zone of
Figure~\ref{fig:zone2} $(x_1,x_2)$ lies, and then apply the corresponding optimal testing sequence. In the specific
case illustrated above, three algorithms out of three were needed to define the zones; however, for any larger scenario,
we will see that only a very small portion of the potential algorithms will be carefully selected.

Our objective is to determine the zones, and the corresponding testing algorithms, for arbitrary $n$, so as to identify which samples in a set are negative and which are not, while minimizing the expected number of testing operations.

\section{Preliminaries}

This section will formalize the notion of a testing procedure, and the cost thereof, so that the problem at
hand can be mathematically described. We aim at the greatest generality, which leads us to introduce `and-tests', a special
case of which are samples that can be pool tested. 

\subsection{Testing procedures}

We consider a collection of $n$ samples. Let $[n]$ denote $\{1, \dotsc, n\}$, and $\Omega = \mathcal{P}({[n]})\backslash\{\emptyset\}$, where $\mathcal{P}$ is the power set (ie. $\mathcal{P}(X)$ is the set of subsets of $X$). 

\begin{definition}[Test]
A \emph{test} is a function $\phi: \Omega \to \{0, 1\}$, that associates a bit to each subset of $\Omega$. 
\end{definition}
We focus in this work on the following:

\begin{definition}[And-Tests]
An \emph{and-test} $\phi: \Omega \to \{0, 1\}$ is a test satisfying the following property:
\begin{equation*}
\forall T \in \Omega, \quad \phi(T) = \bigwedge_{t \in T} \phi(\{t\}).
\end{equation*}
\end{definition}
In other terms, the result of an and-test on a set is exactly the logical and of the test results on individual members of that set.

\begin{remark}Note that \enquote{or-tests}, where $\wedge$ is replaced by $\vee$ in the definition, are exactly dual to our setting. \enquote{xor-tests} can be defined as well but are not investigated here. Although theoretically interesting by their own right, we do not address the situation where both and-tests and or-tests are available, since we know of no concrete application where this is the case.
\end{remark}
Elements of $\Omega$ can be interpreted as $n$-bit strings, with the natural interpretation where the $i$-th bit indicates whether $i$ belongs to the subset. We call \emph{selection} an element of $\Omega$.

\begin{definition}[Outcome]
The \emph{outcome} $F_\phi(T)$ of a test $\phi$ on $T \in \Omega$ is the string of individual test results:
\begin{equation*}
F_\phi(T) = \{ \phi(x), x \in T \} \in \{0,1\}^n.
\end{equation*}
When $T=[n]$, $F_\phi$ will concisely denote $F_\phi([n])$.
\end{definition}
Our purpose is to determine the outcome of a given test $\phi$, by minimizing in the expected number of queries to $\phi$. Note that this minimal expectation is trivially upper bounded by $n$.

\begin{definition}[Splitting]
Let $T \in \Omega$ be a selection and $\phi$ be a test. Let $\mathcal S$ be a subset of $\Omega$.
The \emph{positive part of $\mathcal S$ with respect to $T$}, denoted $\mathcal S_T^\top$, is defined as the set
\begin{equation*}
\mathcal{S}_T^{\top} = \left\{ S | S \in \mathcal{S}, S \wedge T = T \right\}.
\end{equation*}
where the operation $\wedge$ is performed element-wise. This splits $\mathcal S$ into two.
Similarly the complement $\mathcal{S}_T^{\bot} = \mathcal S - \mathcal{S}_T^{\top}$ is called the \emph{negative part of $\mathcal{S}$ with respect to $T$}.
\end{definition}
We are interested in algorithms that find $F_\phi$. More precisely, we focus our attention on the following:
\begin{definition}[Testing procedure]
\label{def5}
A \emph{testing procedure} is a binary tree $\mathcal T$ with labeled nodes and leaves, such that:
\begin{enumerate}
\item The leaves of $\mathcal T$ are in one-to-one correspondence with $\Omega$ in string representation;
\item Each node of $\mathcal T$ which is not a leaf has exactly two children, $(S_\bot, S_\top)$, and is labeled $(S, T)$ where $S \subseteq \Omega$ and $T \in \Omega$, such that
\begin{enumerate}
\item $S_\bot \cap S_\top = \emptyset$
\item $S_\bot \sqcup S_\top$ = S
\item $S_\bot = S_T^\bot$ and $S_\top = S_T^\top$.
\end{enumerate}
\end{enumerate}
\end{definition}
\begin{remark}
It follows from the definition \ref{def5} that a testing procedure is always a \emph{finite} binary tree, and that no useless calls to $\phi$ are performed. Indeed, doing so would result in an empty $S$ for one of the children nodes. Furthermore, the root node has $S = \Omega$.
\end{remark}

\subsection{Interpreting and representing pooling procedures}

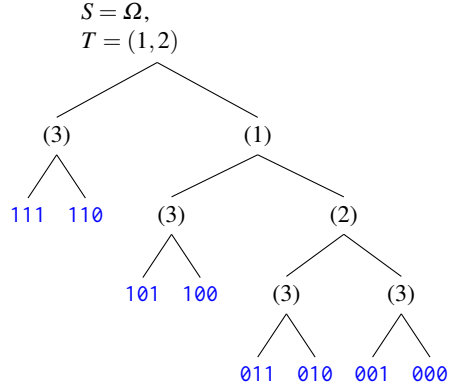
\begin{figure}[!ht]
\centering
\begin{tikzpicture}[
level/.style={sibling distance=2cm,
level distance=0.5cm},
event/.style={text width=2cm,anchor=south}
]
\Tree[
.\node[event]{$S = \Omega$, \\$T = (1,2)$};
	\edge ; [.\node{(3)} ; 
		\edge ; [.\node[blue]{\texttt{111}} ; ] 
		\edge ; [.\node[blue]{\texttt{110}} ; ] 
	]
	\edge ; [.\node{(1)}; 
		\edge ; [.\node{(3)} ;
        	\edge ; [.\node[blue]{\texttt{101}} ; ] 
			\edge ; [.\node[blue]{\texttt{100}} ; ]
        ]
		\edge ; [.\node{(2)} ;
        	\edge ; [.\node{(3)} ;
        		\edge ; [.\node[blue]{\texttt{011}} ; ] 
				\edge ; [.\node[blue]{\texttt{010}} ; ]
            ]
            \edge ; [.\node{(3)} ;
        		\edge ; [.\node[blue]{\texttt{001}} ; ] 
				\edge ; [.\node[blue]{\texttt{000}} ; ]
            ]
		]
	]
]
\end{tikzpicture}
\caption{Graphical representation of a testing procedure. The collection is $[3] = \{1, 2, 3\}$, $\Omega = \{\{1\}, \{2\}, \{3\}, \{1,2\}, \{1,3\}, \{2,3\}, \{1,2,3\}\}$, the initial set of selections is $S = \Omega$. Only the $T$ labels are written on nodes. Only the $S$ labels are written for leaves.}
\label{fig:procedure}
\end{figure}

Consider a testing procedure $\mathcal T$, defined as above. $\mathcal{T}$ describes the following algorithm. At each node $(S, T)$, perform the test $\phi$ on the selection $T$ of samples. If $\phi(T) = 0$, go to the left child; otherwise go to the right child. 
Note that at each node of a testing procedure, only one invocation of $\phi$ is performed.

The tree is finite and thus this algorithm reaches a leaf $S_\text{final}$ in a finite number of steps. By design, $S_\text{final} = F_\phi$.

\begin{remark}
From now on, we will fix $\phi$ and assume it implicitly. 
\end{remark}

\begin{remark} We represent a testing procedure graphically as follows: Nodes (in black) are labeled with $T$, whereas leaves (in blue) are labeled with $S$ written as a binary string. This is illustrated in Figure~\ref{fig:procedure} for $n = 3$.

This representation makes it easy to understand how the algorithm unfolds and what are the outcomes: Starting from the root,
each node tells us which entity is tested. If the test is positive, the right branch is taken, otherwise the left branch is taken.
Leaves indicate which samples tested positive and which samples tested negative from now on.\end{remark}

\begin{remark}
The successive steps of a testing procedure can be seen as imposing new logical constraints. These constraints ought to be satisfiable (otherwise one set $S$ is empty in the tree, which cannot happen). The formula at a leaf is maximal in the sense that any additional constraint would make the formula unsatisfiable. This alternative description in terms of satisfiability of Boolean clauses is in fact strictly equivalent to the one that we gave.

In that case, $T$ is understood as a conjunction $\bigwedge_{T[i] = 1} t_i$, $S$ is a proposition formed by a combination of terms $t_i$, connectors $\vee$ and $\wedge$, and possibly $\neg$. The root has $S = \top$. The left child of a node labeled $(T, S)$ is labeled $S_T^\bot = S \wedge (\neg T)$; while the right child is labeled $S_T^\top = S \wedge T$. At each node and leaf, $S$ must be satisfiable.
\end{remark}

\subsection{Probabilities on trees}

To determine how efficient any given testing procedure is, we need to introduce a probability measure, and 
a metric that counts how many calls to $\phi$ are performed.

We consider the discrete probability space $(\Omega,\Pr)$. The \emph{expected value} of a random variable $X$ is 
classically defined as:
\begin{equation*}
\operatorname{E} [X]  = \sum_{\omega \in \Omega} X(\omega) \Pr(\omega)
\end{equation*}
Let $\mathcal T$ a testing procedure, and let $S \in \Omega$ be one of its leaves. The \emph{length} $\ell_{\mathcal T}(S)$ of $\mathcal T$ over $S$ is the distance on the tree from the root of $\mathcal T$ to the leaf $S$. This corresponds to the number of tests required to find $S$ if $S$ is the outcome of $\phi$. The \emph{expected length} of a testing procedure $\mathcal T$ is defined naturally as:
\begin{equation*}
L_{\mathcal T} = \operatorname{E}\left[ \ell_{\mathcal T} \right] = \sum_{\omega \in \Omega} \ell_{\mathcal T}(\omega) \Pr(\omega)
\end{equation*}
It remains to specify the probabilities $\Pr(\omega)$, i.e. for any given binary string $\omega$, the probability that $\omega$ is the outcome. 

If the different tests are independent, we can answer this question directly with the following result:
\begin{lemma}
Assume that the events \enquote{$\phi(\{i\}) = 1$} and \enquote{$\phi(\{j\}) = 1$} are independent for $i \neq j$. Then, $ \forall \omega \in \Omega$, $\Pr(\omega)$ can be written as a product of monomials of degree 1 in $x_1, \dotsc, x_n$, where 
\begin{equation*}
x_i = \Pr(\phi(\{i\}) = 1) = \Pr(\text{$i$-th bit of }\omega = 1).
\end{equation*}
Thus $L_{\mathcal T}$ is a multivariate polynomial of degree $n$ with integer coefficients.
\end{lemma}
In fact, or-tests provide inherently independent tests. Therefore we will safely assume that the
independence assumption holds.
\begin{example}
Let $n = 5$ and $\omega = \texttt{11101}$, then $\Pr(\omega) = x_1x_2x_3(1-x_4)x_5$.
\end{example}
\begin{remark}
$L_{\mathcal T}$ is uniquely determined as a polynomial by the integer vector of length $2^n$ defined by
all its lengths: $\ell(\mathcal T) = (\ell_{\mathcal T}(\mathtt{0...0}), \dotsc, \ell_{\mathcal T}(\mathtt{1...1}))$.
\end{remark}

\section{Optimal pool tests}

We have now introduced everything necessary to state our goal mathematically.
Our objective is to identify the best performing testing procedure $\mathcal T$ (i.e. having the smallest $L_{\mathcal T}$) in a given situation, i.e. knowing $\Pr(\omega)$ for all $\omega \in \Omega$.

\subsection{Generating all procedures}
We can now explain how to generate all the testing procedures for a given $n \geq 2$.

One straightforward method is to implement a generation algorithm based on the definition of a testing procedure. Algorithm~\ref{alg:gen} does so recursively by using a coroutine. The complete list of testing procedures is recovered by calling $\texttt{FindProcedure}(\Omega, \Omega \setminus \{\emptyset\} )$.

\newalg%
{FindProcedure\label{alg:gen}}%
{$S\in\Omega$, $C\in \Omega$.}%
{A binary tree.}%
{\begin{enumerate}
\item if $|S| == 1$ then return $S$
\item $S_\bot' = S_\top' = C' = \emptyset$
\item for each $c \in C$
\item \quad  $S_\bot = S_c^\bot$ 
\item \quad  $S_\top = S_c^\top$
\item \quad if $S_\bot \notin S_\bot'$ and $S_\top \notin S_\top'$
\item \qquad  $S_\bot' = S_\bot' \cup \{S_\bot\}$
\item \qquad $S_\top' = S_\top' \cup \{S_\top\}$ 
\item \qquad $C' = C' \cup \{c\}$
\item for $i \in \{1, \dotsc, |C'|\}$
\item \quad $\overline C = C - C'[i]$ 
\item \quad for each $\mathcal T_\bot \in \operatorname{FindProcedure}(S_\bot'[i], \overline{C})$
\item \qquad for each $\mathcal T_\top \in \operatorname{FindProcedure}(S_\top'[i], \overline{C})$
\item \qquad \quad yield $(C'[i], \mathcal T_\bot, \mathcal T_\top)$
\end{enumerate}}

We implemented this algorithm (in Python, source code available upon request). The result of testing procedure generations for small values of $n$ is summarized in Table~\ref{tab:gen}. The number of possible testing procedures grows very quickly with $n$.

\begin{table}[!ht]
\centering\small
\caption{Generation results for some small $n$}\label{tab:gen}
\begin{tabular}{crr}\toprule 
$\quad n\quad$ & Number of procedures &  Time \\\midrule 
1 & 1 & 0\\
2 & 4 & $\sim 0$\\
3 & 312 & $\sim0$\\
4 & 36585024 & \quad $\sim30$ mn \\\bottomrule 
\end{tabular}
\end{table}

An informal description of Algorithm~\ref{alg:gen} is the following. Assuming that you have an unfinished procedure (i.e. nodes at the end of branches are not all leaves). For those nodes $S$, compute for each $T$ the sets $S_T^{\top}$ and $S_T^{\bot}$. If either is empty, abort. Otherwise, create a new (unfinished) procedure, and launch recursively on nodes (not on leaves, which are such that $S$ has size $1$).

Algorithm~\ref{alg:gen} terminates because it only calls itself with strictly smaller arguments. We will discuss this algorithm further after describing some properties of the problem at hand.

\subsection{Metaprocedures}

Once the optimality zones, and the corresponding testing procedures, have been identified, it is easy to
write an algorithm which calls the best testing procedure in every scenario. At first sight, it may seem that
nothing is gained from doing so --- but as it turns out that only a handful of procedures need to be implemented.

This construction is captured by the following definition:
\begin{definition}[Metaprocedure]
A \emph{metaprocedure} $\mathcal M$ is a collection of pairs $(Z_i, \mathcal T_i)$ such that:
\begin{enumerate}
\item $Z_i \subseteq [0, 1]^n$, $Z_i \cap Z_j = \emptyset$ whenever $i \neq j$ and $\bigsqcup_i Z_i = [0,1]^n$.
\item $\mathcal T_i$ is a testing procedure and for any testing procedure $\mathcal T$, 
\begin{equation*}
\forall x \in Z_i, \quad L_{\mathcal T_i}(x) \leq L_{\mathcal T}(x).
\end{equation*}
\end{enumerate}
A metaprocedure is interpreted as follows: Given $x \in [0, 1]^n$ find the unique $Z_i$ that contains $x$ and run the corresponding testing procedure $\mathcal T_i$. We extend the notion of expected length accordingly:$
L_{\mathcal M} = \min_{i} L_{\mathcal T_i} \leq n
$.
\end{definition}

One way to find the metaprocedure for $n$, is to enumerate all the testing procedures using Algorithm~\ref{alg:gen}, compute all expected lengths $L_{\mathcal T}$ from the tree structure, and solve polynomial inequalities.

Surprisingly, a vast majority of the procedures generated are nowhere optimal: This is illustrated in Table~\ref{tab:meta}. Furthermore, amongst the remaining procedures, there is a high level of symmetry. For instance, in the case $n = 3$, 8 procedures appear 6 times, 1 a procedure appears 3 times, and 1 procedure appears once. The only difference between the occurrences of these procedures --- which explains why we count them several times --- is the action of the symmetric group $S_6$ on the cube (see Appendix~\ref{app:sym} for a complete description).

\iflongversion
The metaprocedure for $n=3$ cuts the unit cube into 52 zones, which correspond to a highly symmetric and intricate partition, as illustrated in Figures~\ref{fig:meta3}, \ref{fig:meta3b}, and \ref{fig:meta3c}. An STL model was constructed and is available upon request.
\else 
The metaprocedure for $n=3$ cuts the unit cube into 52 zones, which correspond to a highly symmetric and intricate partition, as illustrated in Figure~\ref{fig:meta3} (3D figures in the full version of this paper). An STL model was constructed and is available upon request.
\fi

\begin{figure}[!ht]
\hspace{-1.5cm}
\centering
\includegraphics[height=4.5cm]{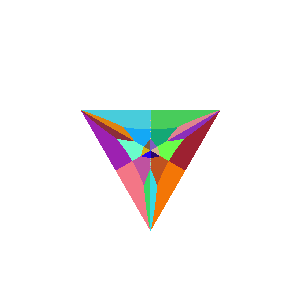}%
\includegraphics[height=4.5cm]{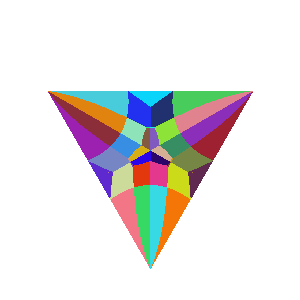}%
\includegraphics[height=4.5cm]{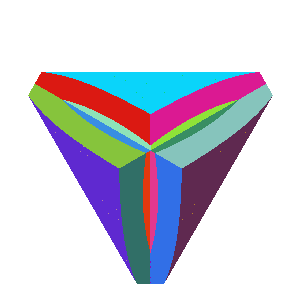} 
\includegraphics[height=4.5cm]{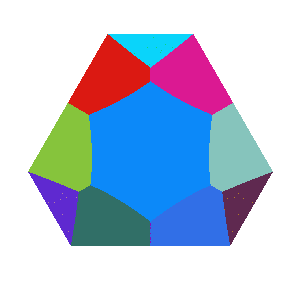}%
\includegraphics[height=4.5cm]{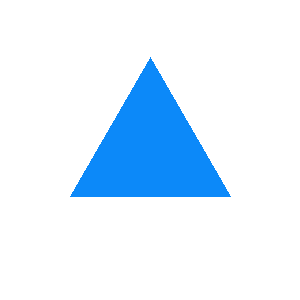}%
\caption{Slices of the cube decomposition for the $n=3$ metaprocedure, each colour corresponds to a different strategy, which is optimal at this position. The slices are taken orthogonally to the cube's main diagonal, with the origin at the center of each picture. Each color corresponds to a procedure. The symmetries are particularly visible.}\label{fig:meta3}
\end{figure}

\iflongversion
\begin{figure}[p]
\centering
\includegraphics[height=5cm]{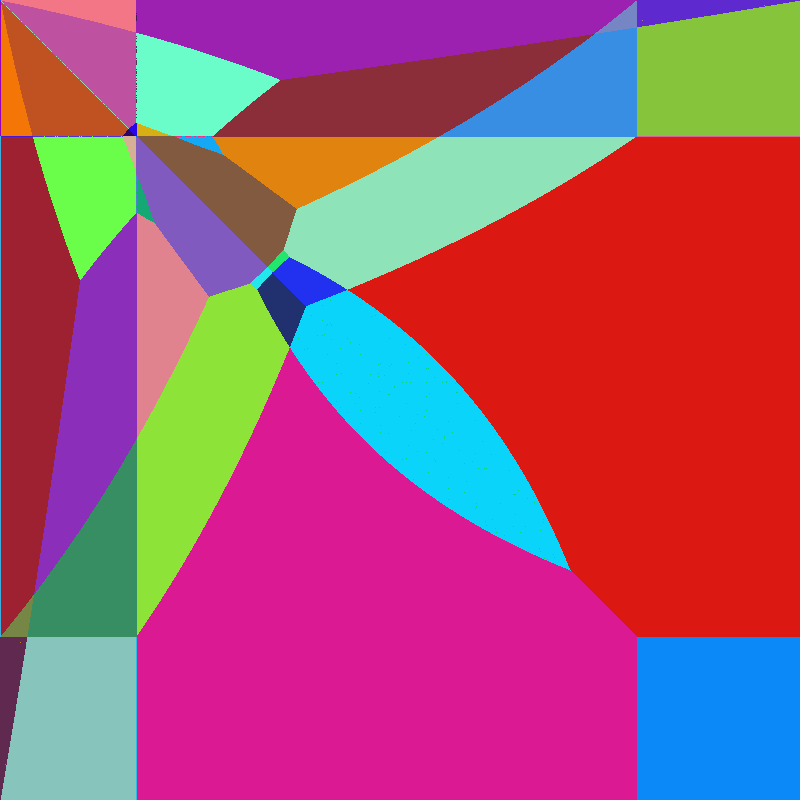}\qquad 
\includegraphics[height=5cm]{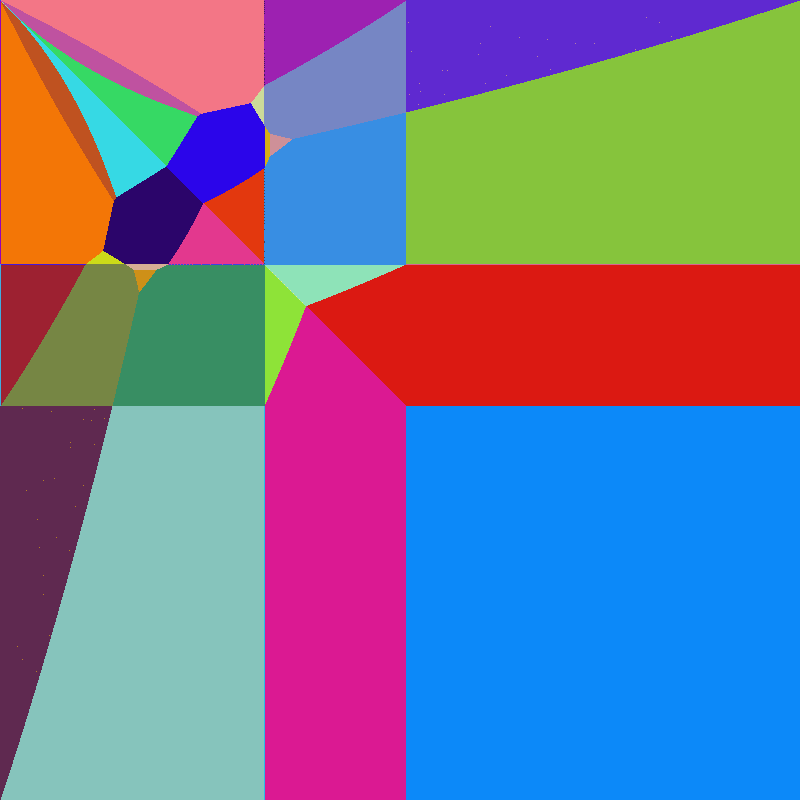}
\caption{Slices through the cube at the $z=0.17$ (left) and the $z=0.33$ (right) planes, showing the metaprocedure's rich structure. Each colour corresponds to a different strategy, which is optimal at this position. The origin is at the top left.}\label{fig:meta3b}
\end{figure}

\begin{figure}[p]
\centering 
\includegraphics[height=4.5cm]{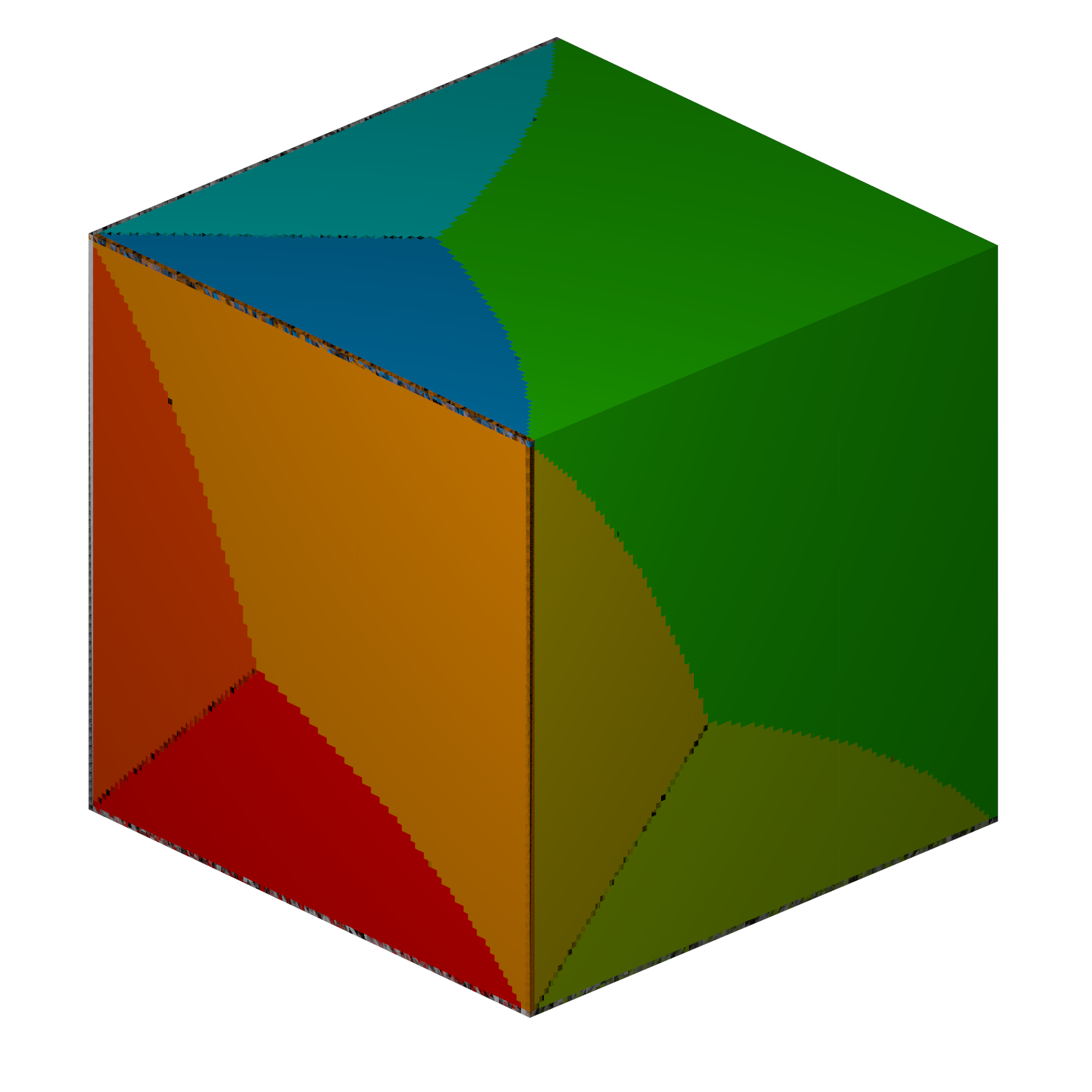}\quad 
\includegraphics[height=5cm]{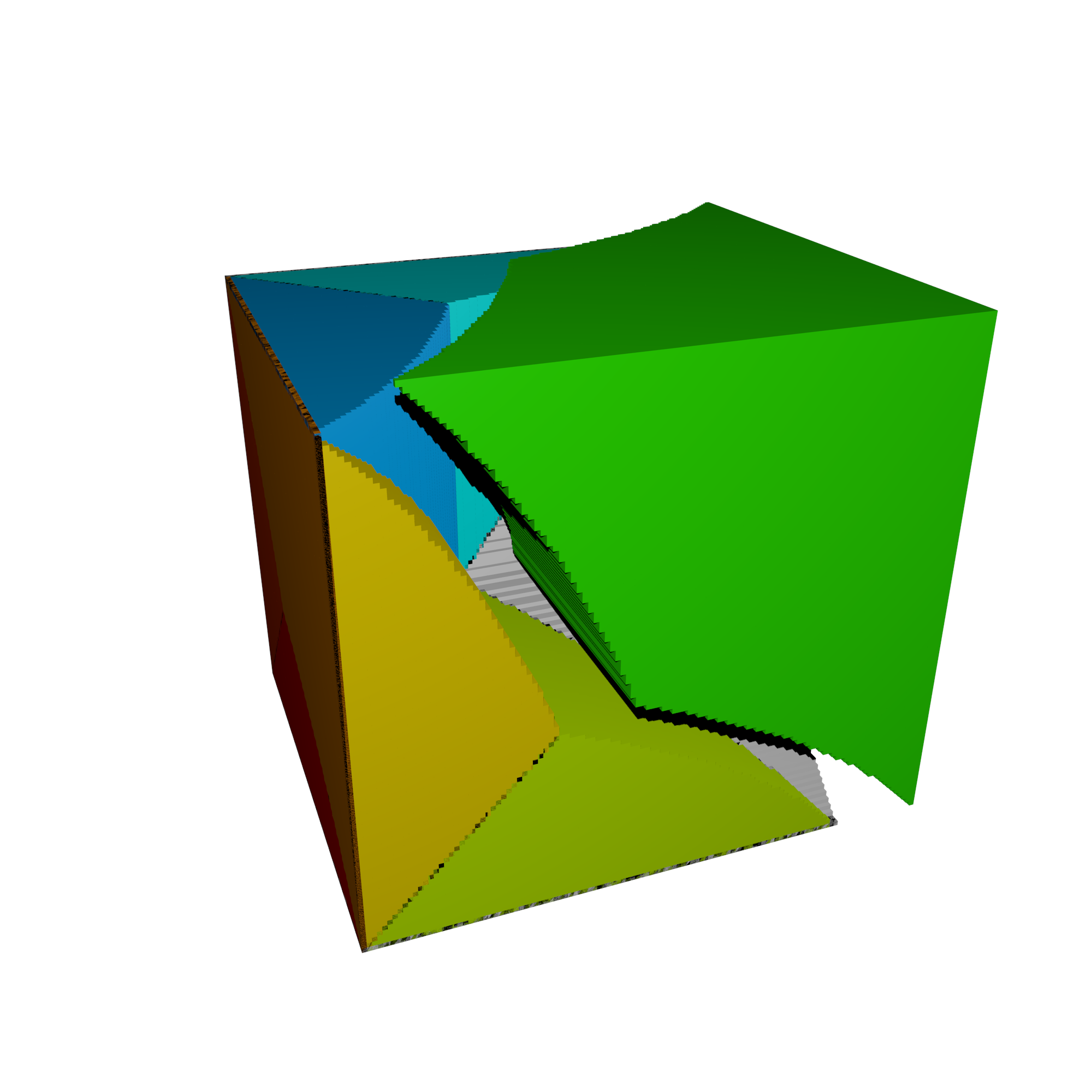}\quad
\includegraphics[height=5cm]{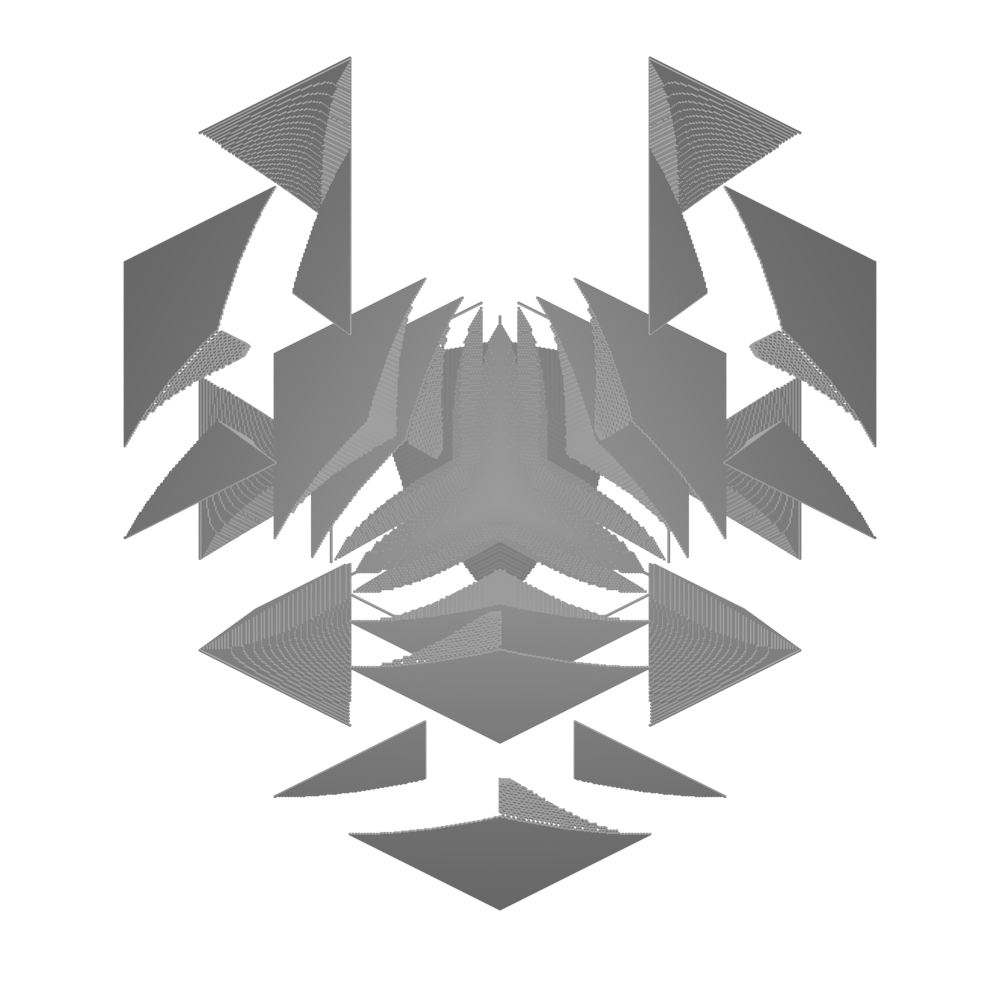}
\caption{A 3D visualisation of the cube. Left: exterior, where it is visible that each face has the same decomposition as the 2D problem; Middle: with the naive algorithm region slightly removed, showing that it accounts for slightly less than half of the total volume; Right: exploded view of the 52 substructures (looking from $(-1, -1, -1)$).}\label{fig:meta3c}
\end{figure}
\fi 

The large number of suboptimal procedures shows that the generate-then-eliminate approach quickly runs out of steam: Generating all procedures for $n=6$ seems out of reach with Algorithm~\ref{alg:gen}\footnote{$n=6$ is still very far from the current \textsc{sars-cov-2} test pooling capacity of $n=32$ or $n=64$ mentioned in the introduction.}. The number of zones, which corresponds to the number of procedures that are optimal in some situation, is on the contrary very reasonable.
\begin{lemma}[Number of naive procedures]
Let $n \geq 1$, then there are
\begin{equation*}
P(n) = \prod_{k=1}^{n}k^{2^{n-k}}
\end{equation*}
equivalent naive procedures.
\end{lemma}
\begin{proof}
By induction on $n$: There are $(n+1)$ choices of a root node, $P(n)$ choices for the left child, and $P(n)$ choices for the right child. This gives the recurrence $P(n+1) = (n+1)P(n)^2$, hence the result. 
\end{proof}
This number grows rapidly and constitutes a lower bound for the total number of procedures (e.g. for $n = 8$ we have $P(n) > 2^{184}$). On the other hand, the naive procedure is the one with maximal multiplicity, which yields a crude upper bound $C_{2^k} P(n)$ on the number of procedures, where $C_t$ is the $t$-th Catalan number defined by:

$$C_t = \frac{1}{t+1}{2t\choose t} \sim \frac{4^n}{n^{3/2}\sqrt{\pi}}$$

\begin{table}[!ht]
\centering \small
\begin{tabular}{crr}\toprule 
$\quad n\quad$ & Number of procedures & \quad Zones \\\midrule 
1 & 1 & 1\\
2 & 4 & 3\\
3 & 312 & 52\\
4 & 36585024 & 181\\
5 & $8.926\cdot{}10^{20}$ & ? \\
6 & $2.242\cdot{}10^{55}$ & ? \\\bottomrule \\
\end{tabular}
\caption{Procedures and metaprocedures for some values of $n$. The number of zones for $n = 5$ and $6$ 
cannot be determined in a reasonable time with the generate-then-eliminate approach.}\label{tab:meta}
\end{table}
The zones can be determined by sampling precisely enough the probability space. Simple arguments about the regularity of polynomials guarantee that this procedure succeeds, when working with infinite numerical precision. In practice, although working with infinite precision is feasible (using rationals), we opted for floating-point numbers, which are faster. The consequence is that sometimes this lack of precision results in incorrect results on the zone borders --- however this is easily improved by increasing the precision or checking manually that there is no sub-zone near the borders.

\section{Pruning the generation tree}

We now focus on some of the properties exhibited by testing procedures, which allows a better understanding of the problem and interesting optimizations. This in effect can be used to prune early the generation of procedures, and write them in a more 
compact way by leveraging symmetries. We consider in this section a testing procedure $\mathcal T$. 

\begin{lemma}\label{prop:len}
Let $B_0$ and $B_1$ be two binary strings of size $n$, that only differ by one bit (i.e. $B_0[i] = 0$ and $B_1[i] = 1$ for some $i$). Then $\ell_{\mathcal T}(B_0) \leq \ell_{\mathcal T}(B_1)$.
\end{lemma}

\begin{proof}
First notice that for all $T$, $T'$, and $b, b' \in \{\top,\bot\}$ we have $(S_T^b)_{T'}^{b'} = (S_{T'}^{b'})_T^b$.
We will denote both by $S_{TT'}^{bb'}$.

We have the following : If there exists $k$, $T_1, \dotsc, T_k$, and $\beta_1, \dotsc, \beta_k$ such that
\begin{equation*}
(\Omega)_{T_1 \cdots T_k}^{\beta_1 \cdots \beta_k} = \{B_1\}
\end{equation*}
then there exists $i \leq k$ such that
\begin{equation*}
(\Omega)_{T_1 \cdots T_i \cdots T_k}^{\beta_1 \cdots \neg \beta_i \cdots \beta_k} = \{B_0\}
\end{equation*}
Indeed there exists $i \leq k$ such that $\beta_i = \top$ and $T_i = \{i_0\} \cup E$ where 
for all $j$ in $E$, $B_0[j] = B_1[j] = 0$. This yields 
\begin{equation*}
(\Omega)_{ T_1 \cdots T_{i-1} T_{i+1} \cdots T_k}^{\beta_1\cdots \beta_{i-1} \beta_{i+1} \cdots \beta_k } = \{B_0,B_1\}
\end{equation*}
and the result follows.
\end{proof}

\begin{remark}
Proposition~\ref{prop:len} indicates that testing procedures are, in general, unbalanced binary trees: The only balanced procedure being the naive one.
\end{remark}

\begin{lemma}
If $\mathcal N$ is the naive procedure, then for any testing procedure $\mathcal T$ and for all
$x_1, \dots, x_n$ such that $x_i > \frac12$,
\begin{equation*}
L_{\mathcal N}\left(x_1, \dotsc, x_n \right) \leq L_{\mathcal T}\left(x_1, \dotsc, x_n \right).
\end{equation*}
In other terms $\{\forall i \in [n], \frac12 \leq x_i \leq 1 \}$ is contained in the naive procedure's optimality zone.
\end{lemma}

\begin{proof}
An immediate corollary of Proposition~\ref{prop:len} is that for all $i \in [n]$, we have $\partial_{x_i} L_{\mathcal T}(x_1, \dotsc, x_n) \geq 0$,
where $\partial_{x_i}$ indicates the derivative with respect to the variable $x_i$. Since the native procedure has a constant
length, it suffices to show that it is optimal at the point $\{\frac12, \dotsc, \frac12 \}$. Evaluating the length polynomials
at this point gives
\begin{equation*}
L_{\mathcal T}\left( \frac12, \dotsc, \frac12\right) = \frac{1}{2^n} \sum_{\omega \in \Omega} \ell_{\mathcal T}(\omega) = \int_{[0, 1]^n} L_{\mathcal T}\, \mathrm dx.
\end{equation*}
Now remember that the naive procedure gives the only perfect tree. It suffices to show that unbalancing this tree in
any way results in a longer sum in the equation above. Indeed, to unbalance the tree one needs to:
\begin{itemize}
\item Remove two bottom-level leaves, turning their root node into a leaf
\item Turn one bottom-level leaf into a node
\item Attach two nodes to this newly-created leaf
\end{itemize}
The total impact on the sum of lengths is $+1$. Hence the naive algorithm is minimal at $\{\frac12, \dotsc, \frac12\}$,
and therefore, in the region $\{\forall i \in [n], \frac12 \leq x_i \leq 1 \}$.
\end{proof}
\begin{remark}
This also show that if we assume that the probabilities are supposed uniform (ie. we assume no a priori knowledge) the optimal procedure is the naive one.  Therefore we can see that the gain for $n=3$  is aprroximately $0,34$ since the optimal procedure in average gives $2,66$.  In percentage the gain is $15\%$. If the probabilites are very low we have a gain of almost $2$  which is $3$ times faster. As expected it is much more interesting if we think that the samples have a good chance to be negative, which is the case in most real life scenarii.
\end{remark}

\begin{lemma}
If the root has a test of cardinality one, then the same algorithms starting at both sons have same expected
stopping time. This applies if the next test is also of cardinal one.
\end{lemma}

\begin{proof}
Without loss of generality we can assume that the test is $\{1\}$. We have 
$\{0,1\}^{n^{\top}}_{\{1\}} = \{0b_2 \cdots b_n |b_2 \cdots b_n \in \{0,1 \}^{n-1}\}$ and
$\{0,1\}^{n^{\bot}}_{\{1\}} = \{1b_2 \cdots b_n |b_2 \cdots b_n \in \{0,1 \}^{n-1}\}$.
A test $T$ that doesn't test $1$ applied on those sets will give the same split for both, and the probability that the test answers yes or no is the same. This is also true for the sets and the tests $T$ such that $i$ is not in $T$ for $i$ in $\{1, \dotsc, k\}$.
$\{0^kb_2 \cdots b_n |b_2 \cdots b_n \in \{0,1 \}^{n-k}\}$ and
$\{01^kb_2 \cdots b_n |b_2 \cdots b_n \in \{0,1 \}^{n-k}\}$.
A test $T$ such that there exists $i$ in  $T$ $\{1, \dotsc, k\}$ brings no information for the set of possibilities $\{01^kb_2 \cdots b_n |b_2 \cdots b_n \in \{0,1 \}^{n-k}\}$, but testing this $i$ is useless for the set $\{0^kb_2 \cdots b_n |b_2 \cdots b_n \in \{0,1 \}^{n-k}\}$. So we can apply the test $T - \{1, \dotsc, k\}$.
\end{proof}

\begin{corollary}
If the root has a test of cardinal one, then an optimal algorithm can always apply the same test for the right and left child. If this test is also of cardinal one then the property is still true.
\end{corollary}
This result helps in identifying redundant descriptions of testing procedures, and can be used to narrow down the generation, by skipping over obvious symmetries of the naive procedure (see Figure~\ref{fig:weird}).

\begin{figure}[!ht]
	\centering\scalebox{0.8}{
\begin{tikzpicture}[level/.style={sibling distance=2cm, level distance=0.5cm}]
\Tree[
.\node{(1)} ;
	\edge ; [.\node{(2)} ; 
		\edge ; [.\node{(3)} ;
        	\edge ; [.\node[blue]{\texttt{111}} ; ]
            \edge ; [.\node[blue]{\texttt{110}} ; ]
        ] 
		\edge ; [.\node{(3)} ;
        	\edge ; [.\node[blue]{\texttt{101}} ; ]
            \edge ; [.\node[blue]{\texttt{100}} ; ]
        ] 
	]
    \edge ; [.\node{(3)} ; 
		\edge ; [.\node{(2)} ;
        	\edge ; [.\node[blue]{\texttt{011}} ; ]
            \edge ; [.\node[blue]{\texttt{001}} ; ]
        ] 
		\edge ; [.\node{(2)} ;
        	\edge ; [.\node[blue]{\texttt{010}} ; ]
            \edge ; [.\node[blue]{\texttt{000}} ; ]
        ] 
	]
]
\end{tikzpicture}}
	\caption{Naive algorithm, where the order of tests are unimportant in the left and right branches.}\label{fig:weird} 
\end{figure}
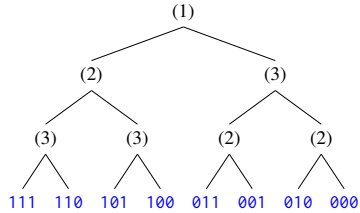

To further accelerate generation we can only keep one representative of each algorithms that have the same expected length for all $x_i$.
\begin{lemma}
If a node labeled $T_1$ has two children that are \emph{both} labeled $T_2$, then we can interchange $T_1$ and $T_2$ without changing the testing procedure's expected length.
\end{lemma}
Yet another simple observation allows to reduce the set of subsets $T$ at each step:
\begin{lemma}
Consider a node labeled $(T,\mathcal S)$. Assume that there is $i \in [n]$ such that, for all $S$ in $\mathcal S$, $i \notin S$. Then we can replace $T$ by $T\cup \{i\}$.
\end{lemma}

\begin{proof}
We can easily see that $S_T^{\top} = S_{T \cup \{i\}}^{\top}$ and $S_T^{\bot} = S_{T \cup \{i\}}^{\bot}$. 
\end{proof}
Finally we can leverage the fact that the solutions exhibit symmetries, which provides both a compact encoding of
testing procedures, and an appreciable reduction in problem size.
\begin{lemma}\label{lem:sym}

Let $\sigma \in \mathfrak{S}_n$ be a permutation on $n$ elements. If we apply $\sigma$ to each node and leaf of $\mathcal T$, which we can write $\sigma(\mathcal T)$, then
\begin{equation*}
L_{\sigma(\mathcal T)}(x_1, \dotsc, x_n) = L_{\mathcal T}\left( \sigma \left( x_1, \dotsc, x_n \right) \right).
\end{equation*}
\end{lemma}

\begin{proof}
Note that for any $S \in \Omega$ and $T \in \Omega\setminus \{\emptyset\}$ we have $\sigma\left(S_T^{\top}\right) = S_{\sigma(T)}^{\top}$ and
$\sigma\left(S_T^{\bot}\right) = S_{\sigma(T)}^{\bot}$, where $\sigma$ operates on each binary string. It follows that for any leaf $S$, $\ell_{\mathcal T}(S)$ becomes $\ell_{\mathcal T}(\sigma(S))$ under the action of $\sigma$, hence the result. 
\end{proof}

\begin{lemma}
Let $S$ be a simplex of the hypercube, $\mathcal T$ a procedure, $E = \{\sigma(\mathcal T) | \sigma \in \mathfrak S_n\}$, then there exists $\mathcal T_0$ in $E$, such that for all $x$ in $S$, $\mathcal T_1$ in $E$ we have 
\begin{equation*}
L_{\mathcal T_0}(x) \leq L_{\mathcal T_1}(x).
\end{equation*}
Moreover we have for all $\sigma $ in $\mathfrak S_n$, $x$ in $\sigma (S)$, $\mathcal T_1$ in $E$
\begin{equation*}
L_{\sigma(\mathcal T_0)}(x) \leq L_{\mathcal T_1}(x).
\end{equation*}

\end{lemma}
\begin{remark}
The last two propositions allow us to solve the problem on a simplex of the hypercube (of volume $1/n!$) such as 
$\{p_1, \dotsc, p_n \mid 1 \geq p_1 \cdots \geq p_n \geq 0\}$.
\end{remark}

\section{Best testing procedure at a point}
\label{app:select}

We examine the following problem: Find the testing procedure $\mathcal{T}$ for a given $k \leq n$, $(p_{i_1}, \dotsc, p_{i_k}) \in [0,1]^n$, and a selection $P \subseteq 2^{[k]}$ that satisfies:
\begin{itemize}
\item{$\mathcal{S}_{\mathcal{T}} = P$,}
\item{$\mathcal{T}$ is optimal at point $(p_{i_1}, \dotsc, p_{i_k})$}
\end{itemize}
This can be computed using a dynamic programming technique, by examining the outcome of each possible test that is the root node of the testing procedure $\mathcal{T}$, which gives Algorithm~\ref{alg:best_at_point}.

The same dynamic programming algorithm can also be used to compute the number of testing procedures (including those leading to duplicate polynomials) that exist in a given dimension. It is actually even easier, since there is a huge number of symmetries that can be exploited to count.\footnote{Indeed, we can apply the algorithm to an even higher dimension than our solution to the given point problem.}

\begin{definition}[Decided point]
We say that $x$ is a \emph{decided point} for $\mathcal{S}$ a set of selections if either of the following is true:
\begin{itemize}
\item{$x \in S$ for all $S \in \mathcal{S}$}
\item{$x \not\in S$ for all $S \in \mathcal{S}$}
\end{itemize}
In the first case, we will say that $x$ is a \emph{positive decided point}, and a \emph{negative decided point} in the second case.

We denote by $\mathcal{D}_{\mathcal{S}}^+$ the set of positive decided points of $\mathcal{S}$, $\mathcal{D}_{\mathcal{S}}^-$ its set of negative decided points, and $\mathcal{D}_{\mathcal{S}} = \mathcal{D}_{\mathcal{S}}^+ \cup \mathcal{D}_{\mathcal{S}}^-$ its set of decided points.
\end{definition}

\newalg%
{FindOptimal\label{alg:best_at_point}}%
{$k \geq 0$, $(p_1, \dots, p_k) \in [0, 1]^k$, $\mathcal{S} \subset 2^{[k]}$.}%
{The optimal testing procedure $\mathcal{T}$ at point $(p_1, \dots, p_k)$ which satisfies $\mathcal{S}_{\mathcal{T}} = \mathcal{S}$.}%
{\begin{enumerate}
\item if $k == 0$ then return the naive algorithm
\item if $|\mathcal D_{\mathcal S}| > 0$
\item \quad $U \gets \{u_1, \dotsc, u_\ell\} = [k] \setminus \mathcal D_{\mathcal S}$ 
\item \quad $ \mathcal{R} \gets \{ \{ r_1, \dotsc, r_p \} \mid  \{ u_{r_1}, \dots, u_{r_p} \} \cup \mathcal{D}_{\mathcal{S}}^+ \} $
\item \quad $\mathcal{T} \gets \operatorname{FindOptimal}\left(\ell, (p_{u_1}, \dots, p_{u_\ell}),\mathcal{R}\right)$
\item \quad replace $\{t_1, \dotsc, t_r\}$ by $\{u_{t_1}, \dotsc, u_{t_r}\}$ in $\mathcal T$
\item \quad replace $\{\ell_1, \dotsc, \ell_r \}$ by $\{u_{\ell_1}, \dotsc, u_{\ell_r}\} \cup \mathcal{D}_{\mathcal S}^+$ in $\mathcal T$ 
\item else
\item \quad $W \gets \emptyset$
\item \quad for each $T \subseteq [k]$
\item \qquad   $\mathcal{S}_\bot \gets \mathcal{S}_T^\bot$ 
\item \qquad   $\mathcal{S}_\top \gets \mathcal{S}_T^\top$ 
\item \qquad if $S_\bot = \emptyset$ or $S_\top = \emptyset$ then continue
\item \qquad  $\mathcal{T}_\bot \gets \operatorname{FindOptimal}(k, (p_1, \dotsc, p_k), \mathcal{S}_\bot)$ 
\item \qquad $\mathcal{T}_\top \gets \operatorname{FindOptimal}(k, (p_1, \dotsc, p_k), \mathcal{S}_\top)$ \\
\item \qquad $W \gets W \cup \{ (\mathcal{T}, \mathcal{T}_\bot, \mathcal{T}_\top) \}$
\item \quad return the best algorithm in $W$ at point $(p_1, \dotsc, p_n)$
\end{enumerate}}
Counting the number of algorithms in a given dimension works the same way; the only difference is that there is no need to look at the probabilities, and thus, the resulting Algorithm~\ref{alg:count_algs} does fewer recursive calls and is faster.\footnote{We are not aware of a closed-form formula providing the same values
as this algorithm.}
\newalg%
{CountAlgorithms\label{alg:count_algs}}%
{$k \geq 0$, $\mathcal{S} \subset 2^{[k]}$.}%
{The number of testing procedures which satisfy $\mathcal{S}_{\mathcal{T}} = \mathcal{S}$.}%
{\begin{enumerate}
\item if $k == 0$ then return $1$
\item if $|\mathcal D_\mathcal{S}| > 0$
\item \quad $U \gets \{u_1, \dotsc, u_\ell\} = [k] \setminus \mathcal D_\mathcal{S}$
\item \quad $ \mathcal{R} = \{ \{ r_1, \dots, r_p \} \mid \{ u_{r_1}, \dots, u_{r_p} \} \cup \mathcal{D}_{\mathcal{S}}^+ \} $ 
\item \quad return $\operatorname{CountAlgorithms}(\ell, \mathcal R)$
\item $c \gets 0$
\item for each $T \subseteq [k]$
\item \quad $\mathcal{S}_\bot \gets \mathcal{S}_T^\bot$ 
\item \quad $\mathcal{S}_\top \gets \mathcal{S}_T^\top$ 
\item \quad if $S_\bot = \emptyset$ or $S_\top = \emptyset$ then continue
\item \quad $c_\bot \gets \operatorname{CountAlgorithms}(k, (p_1, \dotsc, p_k), \mathcal{S}_\bot)$ 
\item \quad $c_\top \gets \operatorname{CountAlgorithms}(k, (p_1, \dotsc, p_k), \mathcal{S}_\top)$ 
\item \quad $c \gets c + c_\top c_\bot$
\item return $c$
\end{enumerate}}


\iflongversion
\section{Enumerating procedures for $n=3$}

All the procedures for $n=3$ that are optimal at some point, up to symmetries, are represented in Figure~\ref{fig:n3complete}.

\begin{figure}
\centering 
\begin{tikzpicture}[scale=0.6]
\Tree[
.\node{(1)};
	\edge ; [.\node{(2)} ; 
		\edge ; [.\node{(3)} ;
        	\edge ; [.\node[blue]{\texttt{111}} ; ]
            \edge ; [.\node[blue]{\texttt{110}} ; ]
        ]
		\edge ; [.\node{(3)} ;
        	\edge ; [.\node[blue]{\texttt{101}} ; ]
            \edge ; [.\node[blue]{\texttt{100}} ; ]
		] 
	]
	\edge ; [.\node{(2)}; 
		\edge ; [.\node{(3)} ;
        	\edge ; [.\node[blue]{\texttt{011}} ; ]
            \edge ; [.\node[blue]{\texttt{010}} ; ]
        ]
		\edge ; [.\node{(3)} ;
        	\edge ; [.\node[blue]{\texttt{001}} ; ]
            \edge ; [.\node[blue]{\texttt{000}} ; ]
		]
	]
]
\end{tikzpicture}
\begin{tikzpicture}[scale=0.6]
\Tree[
.\node{(1,2)};
	\edge ; [.\node{(3)} ; 
        \edge ; [.\node[blue]{\texttt{111}} ; ]
        \edge ; [.\node[blue]{\texttt{110}} ; ] 
	]
	\edge ; [.\node{(1,3)}; 
		\edge ; [.\node[blue]{\texttt{101}} ; ]
		\edge ; [.\node{(1)} ;
        	\edge ; [.\node[blue]{\texttt{100}} ; ]
            \edge ; [.\node{(2)} ;
                \edge ; [.\node{(3)} ;
                    \edge ; [.\node[blue]{\texttt{011}} ; ]
                    \edge ; [.\node[blue]{\texttt{010}} ; ]
                ]
                \edge ; [.\node{(3)} ;
                    \edge ; [.\node[blue]{\texttt{001}} ; ]
                    \edge ; [.\node[blue]{\texttt{000}} ; ]
                ]
            ]
		]
	]
]
\end{tikzpicture}
\begin{tikzpicture}[scale=0.6]
\Tree[
.\node{(1,2)};
	\edge ; [.\node{(3)} ; 
        \edge ; [.\node[blue]{\texttt{111}} ; ]
        \edge ; [.\node[blue]{\texttt{110}} ; ] 
	]
	\edge ; [.\node{(1)}; 
		\edge ; [.\node{(3)} ;
            \edge ; [.\node[blue]{\texttt{101}} ; ]
            \edge ; [.\node[blue]{\texttt{100}} ; ]
        ]
		\edge ; [.\node{(2)} ;
        	\edge ; [.\node{(3)} ; 
                \edge ; [.\node[blue]{\texttt{011}} ; ]
                \edge ; [.\node[blue]{\texttt{010}} ; ]
            ]
            \edge ; [.\node{(3)} ;
                \edge ; [.\node[blue]{\texttt{001}} ; ]
                \edge ; [.\node[blue]{\texttt{000}} ; ]
            ]
		]
	]
]
\end{tikzpicture}
\begin{tikzpicture}[scale=0.6]
\Tree[
.\node{(1,2,3)};
	\edge ; [.\node[blue]{\texttt{111}} ; ]
	\edge ; [.\node{(1,2)}; 
		\edge ; [.\node[blue]{\texttt{110}} ; ]
		\edge ; [.\node{(1)} ;
        	\edge ; [.\node{(3)} ;
                \edge ; [.\node[blue]{\texttt{101}} ; ]
                \edge ; [.\node[blue]{\texttt{100}} ; ]
            ]
            \edge ; [.\node{(2,3)} ;
                \edge ; [.\node[blue]{\texttt{011}} ; ]
                \edge ; [.\node{(2)} ;
                    \edge ; [.\node[blue]{\texttt{010}} ; ]
                    \edge ; [.\node{(3)} ;
                        \edge ; [.\node[blue]{\texttt{001}} ; ]
                        \edge ; [.\node[blue]{\texttt{000}} ; ]
                    ]
                ]
            ]
		]
	]
]
\end{tikzpicture}
\begin{tikzpicture}[scale=0.6]
\Tree[
.\node{(1,2,3)};
	\edge ; [.\node[blue]{\texttt{111}} ; ]
	\edge ; [.\node{(1,2)}; 
		\edge ; [.\node[blue]{\texttt{110}} ; ]
		\edge ; [.\node{(1,3)} ;
        	\edge ; [.\node[blue]{\texttt{101}} ; ]
            \edge ; [.\node{(2,3)} ;
                \edge ; [.\node[blue]{\texttt{011}} ; ]
                \edge ; [.\node{(1)} ;
                    \edge ; [.\node[blue]{\texttt{100}} ; ]
                    \edge ; [.\node{(2)} ;
                        \edge ; [.\node[blue]{\texttt{010}} ; ]
                        \edge ; [.\node{(3)} ;
                            \edge ; [.\node[blue]{\texttt{001}} ; ]
                            \edge ; [.\node[blue]{\texttt{000}} ; ]
                        ]
                    ]
                ]
            ]
		]
	]
]
\end{tikzpicture}
\begin{tikzpicture}[scale=0.6]
\Tree[
.\node{(1,3)};
	\edge ; [.\node{(2)} ; 
        \edge ; [.\node[blue]{\texttt{111}} ; ]
        \edge ; [.\node[blue]{\texttt{101}} ; ] 
	]
	\edge ; [.\node{(2,3)}; 
		\edge ; [.\node[blue]{\texttt{011}} ; ]
		\edge ; [.\node{(1)} ;
        	\edge ; [.\node{(2)} ;
                \edge ; [.\node[blue]{\texttt{110}} ; ]
                \edge ; [.\node[blue]{\texttt{100}} ; ]
            ]
            \edge ; [.\node{(3)} ;
                \edge ; [.\node[blue]{\texttt{001}} ; ]
                \edge ; [.\node{(2)} ;
                    \edge ; [.\node[blue]{\texttt{100}} ; ]
                    \edge ; [.\node[blue]{\texttt{000}} ; ]
                ]
            ]
		]
	]
]
\end{tikzpicture}
\begin{tikzpicture}[scale=0.6]
\Tree[
.\node{(1,2)};
	\edge ; [.\node{(3)} ;
        \edge ; [.\node[blue]{\texttt{111}} ; ]
        \edge ; [.\node[blue]{\texttt{110}} ; ]
    ]
	\edge ; [.\node{(1,3)}; 
		\edge ; [.\node[blue]{\texttt{101}} ; ]
		\edge ; [.\node{(1)} ;
        	\edge ; [.\node[blue]{\texttt{100}} ; ]
            \edge ; [.\node{(2,3)} ;
                \edge ; [.\node[blue]{\texttt{011}} ; ]
                \edge ; [.\node{(2)} ;
                    \edge ; [.\node[blue]{\texttt{010}} ; ]
                    \edge ; [.\node{(3)} ;
                        \edge ; [.\node[blue]{\texttt{001}} ; ]
                        \edge ; [.\node[blue]{\texttt{000}} ; ]
                    ]
                ]
            ]
		]
	]
]
\end{tikzpicture}
\begin{tikzpicture}[scale=0.6]
\Tree[
.\node{(1,2,3)};
	\edge ; [.\node[blue]{\texttt{111}} ; ]
	\edge ; [.\node{(1,2)}; 
		\edge ; [.\node[blue]{\texttt{110}} ; ]
		\edge ; [.\node{(1,3)} ;
        	\edge ; [.\node[blue]{\texttt{101}} ; ]
            \edge ; [.\node{(1)} ;
                \edge ; [.\node[blue]{\texttt{100}} ; ]
                \edge ; [.\node{(2,3)} ;
                    \edge ; [.\node[blue]{\texttt{011}} ; ]
                    \edge ; [.\node{(2)} ;
                        \edge ; [.\node[blue]{\texttt{010}} ; ]
                        \edge ; [.\node{(3)} ;
                            \edge ; [.\node[blue]{\texttt{001}} ; ]
                            \edge ; [.\node[blue]{\texttt{000}} ; ]
                        ]
                    ]
                ]
            ]
		]
	]
]
\end{tikzpicture}
\begin{tikzpicture}[scale=0.6]
\Tree[
.\node{(1,2)};
	\edge ; [.\node{(3)} ;
        \edge ; [.\node[blue]{\texttt{111}} ; ]
        \edge ; [.\node[blue]{\texttt{110}} ; ]
    ]
	\edge ; [.\node{(1)}; 
		\edge ; [.\node{(3)} ;
            \edge ; [.\node[blue]{\texttt{101}} ; ]
            \edge ; [.\node[blue]{\texttt{100}} ; ]
        ]
		\edge ; [.\node{(2,3)} ;
        	\edge ; [.\node[blue]{\texttt{011}} ; ]
            \edge ; [.\node{(2)} ;
                \edge ; [.\node[blue]{\texttt{010}} ; ]
                \edge ; [.\node{(3)} ;
                    \edge ; [.\node[blue]{\texttt{001}} ; ]
                    \edge ; [.\node[blue]{\texttt{000}} ; ]
                ]
            ]
        ]
    ]
]
\end{tikzpicture}
\begin{tikzpicture}[scale=0.6]
\Tree[
.\node{(1,2,3)};
	\edge ; [.\node[blue]{\texttt{111}} ; ]
	\edge ; [.\node{(1)}; 
		\edge ; [.\node{(2)} ;
            \edge ; [.\node[blue]{\texttt{110}} ; ]
            \edge ; [.\node{(3)} ;
                \edge ; [.\node[blue]{\texttt{101}} ; ]
                \edge ; [.\node[blue]{\texttt{100}} ; ]
            ]
        ]
		\edge ; [.\node{(2,3)} ;
        	\edge ; [.\node[blue]{\texttt{011}} ; ]
            \edge ; [.\node{(2)} ;
                \edge ; [.\node[blue]{\texttt{010}} ; ]
                \edge ; [.\node{(3)} ;
                    \edge ; [.\node[blue]{\texttt{001}} ; ]
                    \edge ; [.\node[blue]{\texttt{000}} ; ]
                ]
            ]
        ]
    ]
]
\end{tikzpicture}
\caption{Optimal procedures (without permutations) for each zone when $n=3$.}\label{fig:n3complete}
\end{figure}
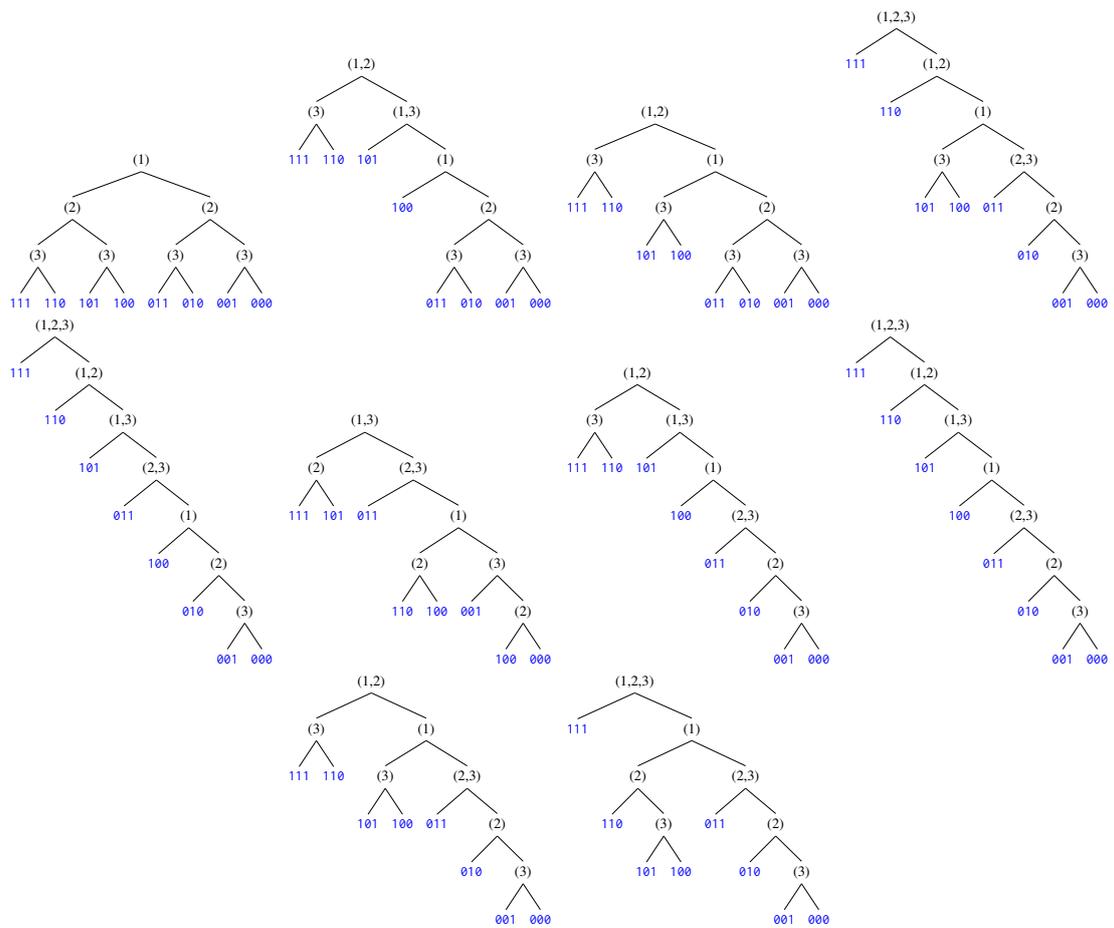

\else 

\fi 

\section{Conclusion and open questions}

\iflongversion
We have introduced the question of optimal pool testing with a priori probabilities, where one is given a set of samples
and must determine in the least average number of operations which samples are negative, and which are not. We formalized this
problem and pointed out several interesting combinatorial and algebraic properties that speed up the computation of
an optimal sequence of operations --- which we call a \emph{metaprocedure}. We determined the exact solution 
for up to $4$ samples.

For larger values, our approach requires too many computation to be tractable, and thus an exact solution is out of reach; however
we gave several heuristic algorithms that scale well. We showed that these heuristics are sub-optimal in all cases,
but they always do better than standard screening. The existence of a polynomial-time algorithm that finds optimal metaprocedures for large value of $n$ is an open question --- although there is probably more hope in finding better heuristics. An alternative would be to modify our generation algorithm to kill branches when the resulting expected lengths are all worse than some already-known procedure.

Once the metaprocedure for a given $n$ is known, which only needs to be computed once, implementation is straightforward and only invokes a handful of (automatically generated) cases.

Finally, in our model we do not consider false positives and false negatives. In other words, tests are assumed to be 100\% accurate. Integrating in the model false positive and false negative probabilities is an interesting research challenge.

Besides the performance gain resulting from implementing metaprocedures for sample testing, the very general framework allows for applications in medical and engineering tests.
\else 
We introduced and analysed the question of optimal pool testing with a priori probabilities. This problem yields several interesting combinatorial and algebraic properties, and we set out to identify the best testing procedures, which is completely solved for $n\geq4$ (see Appendix~\ref{app:sym} for $n=3$). Once the metaprocedure for a given $n$ is known, which only needs to be computed once, implementation is straightforward and only invokes a handful of (automatically generated) cases.

For larger values, an exact solution seems out of reach; however we give several heuristic algorithms in Appendix~\ref{app:approx}. The existence of a polynomial-time algorithm that finds optimal metaprocedures for large value of $n$ is an open question --- although there is probably more hope in finding better heuristics.

Finally, in our model we do not consider false positives and false negatives. In other words, tests are assumed to be 100\% accurate. Integrating in the model false positive and false negative probabilities is an interesting research challenge.

Besides the performance gain resulting from implementing metaprocedures for sample tests, the very general framework allows for applications in medical and engineering tests.
\fi

\bibliographystyle{alpha}
\bibliography{bib/biblio.bib}
\appendix 

\section{Approximation heuristics}\label{app:approx}

The approach consisting in generating many candidates, only to select a few, is wasteful. In fact, for large values of $n$ (even from $10$), generating all the candidates is beyond reach, despite the optimizations we described.

Instead, one would like to obtain the optimal testing procedure \emph{directly}. It is a somewhat simpler problem, and
we can find the solution by improving on our generation-then-selection algorithm (see Appendix~\ref{app:select}). However if
we wish to address larger values of $n$, we must relax the constraints and use the heuristic algorithms described below,
which achieve near-optimal results. This would be usefull in real life scenarii for \covid tests since we would like to test hundreds or more samples to have real gain.

\subsection{Information-Based Heuristic}
We first associate a \enquote{cost} to each outcome $S$, and set of outcomes $\mathcal S$:
\begin{align*}
\operatorname{cost}(S, \mathcal S) & = f(S, \mathcal S) + g(S, \mathcal S)\\
f(S, \mathcal S) & = \#\{i \in [n] \text{ s.t. } s[i] = 1 \text{ and } \exists S \in \mathcal S, S'[i] = 0 \} \\
g(S, \mathcal S) & = 
\begin{cases}
1 & \text{if $\exists i \in \{i \in [n] \text{ s.t. } S[i] =0\}, \exists S' \in \mathcal S, S'[i] = 1$}\\
0 & \text{otherwise}
\end{cases}
\end{align*}
This function approximates the smallest integer $n$ such that there exists $n$ calls to $\phi$ with arguments $T_1, \dotsc, T_n$, and $\beta_1, \dotsc, \beta_n$ in $\{\bot,\top\}$ with $\mathcal S _{T_1 \cdots T_n}^{\beta_1, \dotsc, \beta_n} = \{S\}$.
This function is used to define a \enquote{gain} function evaluating how much information is gathered when performing a test knowing the set of outcomes:
\begin{equation*}
\operatorname{gain}(T, \mathcal S) 
= \sum_{S \in \mathcal S_T^{\top}} \left( 1 - \frac{\operatorname{cost}(S, \mathcal S_T^{\top})}{\operatorname{cost}(S, \mathcal S)} \right) \Pr(S)
	+ \sum_{S \in \mathcal S_T^{\bot}} \left( 1 -  \frac{\operatorname{cost}(S, \mathcal S_T^{\bot})}{\operatorname{cost}(S,S)} \right)  \Pr(S)
\end{equation*}
Intuitively, we give higher gains to subsets $T$ on which testing gives more information. Note that, if a call to $\phi$ doesn't give any information (i.e. $S_T^{\top}$ or $S_T^{\bot}$ is empty), then $\operatorname{gain}(T, S) = 0$.

This heuristic provides us with a greedy algorithm that is straightforward to implement. For given values $x_1, \dotsc, x_n$ we thus obtain a testing procedure $\mathcal T_{H}$. 

\paragraph{Testing the heuristic.} We compared numerically $\mathcal T_H$ to the metaprocedure found by exhaustion in the case $n = 3$. The comparison consists in sampling points at random, and computing the sample mean of each algorithm's length on this input. The heuristic procedure gives a mean of $2.666$, which underperform the optimal procedure ($2.661$) by only 1\%.

\paragraph{Counter-example to optimality.}
In some cases, the heuristic procedure behaves very differently from the metaprocedure. For instance, for $n = 3$, $x_1 = 0.01$, $x_2 = 0.17$, $x_3 = 0.51$, the metaprocedure yields a tree which has an expected length of $1.889$. The heuristic however produces a tree which has expected length $1.96$. Both trees are represented in Figure~\ref{fig:comparison}.

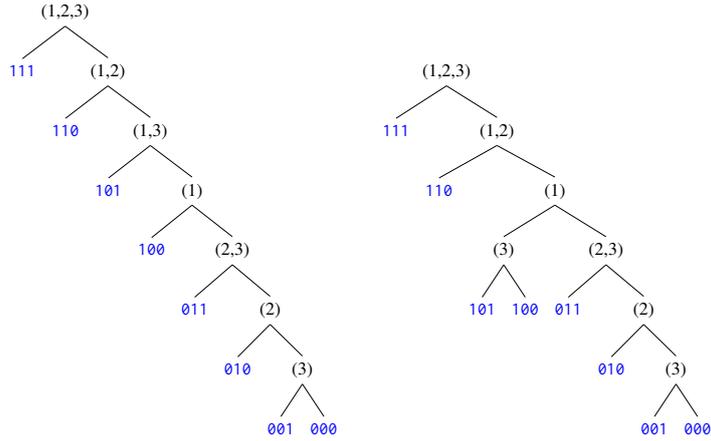
\begin{figure}[!ht]
\centering 
\scalebox{0.75}{
\begin{tikzpicture}[level/.style={sibling distance=3cm/#1, level distance=0.44cm}]
\Tree[
.\node{(1,2,3)} ;
	\edge ; [.\node[blue]{\texttt{111}} ; ] 
	\edge ; [.\node{(1,2)} ;
       	\edge ; [.\node[blue]{\texttt{110}} ; ]
        \edge ; [.\node{(1,3)} ;
           	\edge ; [.\node[blue]{\texttt{101}} ; ]
            \edge ; [.\node{(1)} ;
            	\edge ; [.\node[blue]{\texttt{100}} ; ]
                \edge ; [.\node{(2,3)} ;
                	\edge ; [.\node[blue]{\texttt{011}} ; ]
                    \edge ; [.\node{(2)} ;
                    	\edge ; [.\node[blue]{\texttt{010}} ; ]
                        \edge ; [.\node{(3)} ;
                        	\edge ; [.\node[blue]{\texttt{001}} ; ]
                            \edge ; [.\node[blue]{\texttt{000}} ; ]
                        ]
                    ]
                ]
            ]
        ]
	]
]
\end{tikzpicture}
}\quad 
\scalebox{0.75}{
\begin{tikzpicture}[level/.style={sibling distance=4.5cm/#1, level distance=0.5cm}]
\Tree[
.\node{(1,2,3)} ;
	\edge ; [.\node[blue]{\texttt{111}} ; ] 
	\edge ; [.\node{(1,2)} ;
       	\edge ; [.\node[blue]{\texttt{110}} ; ]
        \edge ; [.\node{(1)} ;
           	\edge ; [.\node{(3)} ;
            	\edge ; [.\node[blue]{\texttt{101}} ; ]
                \edge ; [.\node[blue]{\texttt{100}} ; ]
            ]
            \edge ; [.\node{(2,3)} ;
            	\edge ; [.\node[blue]{\texttt{011}} ; ]
                \edge ; [.\node{(2)} ;
                	\edge ; [.\node[blue]{\texttt{010}} ; ]
                    \edge ; [.\node{(3)} ;
                    	\edge ; [.\node[blue]{\texttt{001}} ; ]
                        \edge ; [.\node[blue]{\texttt{000}} ; ]
                    ]
                ]
            ]
        ]
    ]
]
\end{tikzpicture}
}
\caption{The optimal metaprocedure tree (left), and heuristic metaprocedure (right) for the same point $x = (0.01, 0.17, 0.51)$. The optimal procedure has expected length $1.889$, as compared to $1.96$ for the heuristic procedure.}\label{fig:comparison}
\end{figure}

Beyond their different lengths, the main difference between the two procedures of Figure~\ref{fig:comparison} begin at the third node. At that node the set $S$ is the same, namely $\{\bl{010},\bl{011},\bl{100},\bl{101},\bl{110},\bl{111}\}$, but the two procedures settle for a different $T$:  The metaprocedure splits $S$, with $T = \{1,3\}$, into $S_T^\bot = \{\bl{010}\}$ and $S_T^\top =\{\bl{011},\bl{100},\bl{101},\bl{110},\bl{111}\}$; while the heuristic chooses $T=\{1\}$ instead, and gets $S_T^\bot =\{\bl{010},\bl{011}\}$ and $ S_T^\top =\{\bl{100},\bl{101},\bl{110},\bl{111}\}$.

To understand this difference, first notice that besides $\bl{010}$ and $\bl{011}$, all leaves are associated to a very low probability. The heuristic fails to capture that by choosing $T = \{1,3\}$ early, it could later rule out the leaf $\bl{010}$ in one step and $\bl{011}$ in two. There does not seem to be a simple greedy way to detect this early on.




\subsection{Pairing heuristic}

Another approach is to use small metaprocedures on subsets of the complete problem.
Concretely, given $n$ samples to test, place them at random into $k$-tuples (from some small value $k$, e.g. 5).
Then apply the $k$-metaprocedure on these tuples. While sub-optimal, this approach does not yield worst results
than the naive procedure.

In cases where it makes sense to assume that all the $x_i$ are equal, then we may even recursively use the metaprocedures, i.e. the metaprocedures to be run are themselves places into $k$-tuples, etc. Using lazy evaluation,
only the necessary tests are performed.

\section{Equivalences and symmetries for $n = 3$}\label{app:sym}

A procedure can undergo a transformation that leaves its expected length unchanged. Such transformations are called \emph{equivalences}. On the other hand, Lemma~\ref{lem:sym} shows that some transformations operate a permutation $\sigma$ on the variables $x_i$ --- such transformations are called \emph{symmetries}.

Equivalences and symmetries are responsible for a large part of the combinatorial explosion observed when generating all procedures. By focusing on procedures up to symmetry, we can thus describe the complete set in a more compact way and attempt a first classification.

\iflongversion
In the following representations (Figures~\ref{fig:n3g1}, \ref{fig:n3g2}, and \ref{fig:n3g3}), blue indicates a fixed part, and red indicate a part undergoing some permutation. Double-headed arrows indicate that swapping nodes is possible. The number of symmetries obtained by such an operation is indicated under the curly brace below.
\else 
In the following representations (Figure~\ref{fig:n3g1}, other examples in the full version), blue indicates a fixed part, and red indicate a part undergoing some permutation. Double-headed arrows indicate that swapping nodes is possible. The number of symmetries obtained by such an operation is indicated under the curly brace below.
\fi

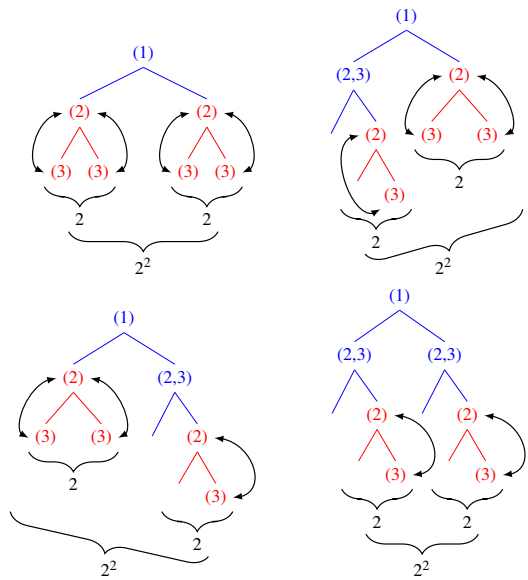
\begin{figure}[!ht]
\centering 
\scalebox{0.75}{
\begin{tikzpicture}[level 1/.style={sibling distance=10mm}]
\Tree[.\node[blue]{(1)} ;
      \edge[blue] ; [.\node(perm2)[red]{(2)} ; 
       \edge[red] ; [.\node(perm21)[red]{(3)} ; ] 
       \edge[red] ; [.\node(perm22)[red]{(3)} ; ] 
      ]
      \edge[blue] ; [.\node(perm1)[red]{(2)}; 
      \edge[red] ; \node(perm11)[red]{(3)};
      \edge[red] ; \node(perm12)[red]{(3)};
      ]
     ]
\draw[semithick,<->,>=latex] (perm11)..controls +(west:0.7) and +(west:0.8)..(perm1);
\draw[semithick,<->,>=latex] (perm12)..controls +(east:0.7) and +(east:0.8)..(perm1);
\draw[semithick,<->,>=latex] (perm21)..controls +(west:0.7) and +(west:0.8)..(perm2);
\draw[semithick,<->,>=latex] (perm22)..controls +(east:0.7) and +(east:0.8)..(perm2);
\draw[semithick,-,decorate,decoration={brace,amplitude=10pt,mirror}] (perm21.south west) -- node(perm31)[below=10pt] {$2$} (perm22.south east);
\draw[semithick,-,decorate,decoration={brace,amplitude=10pt,mirror}] (perm11.south west) -- node(perm32)[below=10pt] {$2$} (perm12.south east);
\draw[semithick,-,decorate,decoration={brace,amplitude=10pt,mirror}] (perm31.south west) -- node[below=10pt] {$2^2$} (perm32.south east);
\end{tikzpicture}}\qquad 
\scalebox{0.75}{
\begin{tikzpicture}
\Tree[.\node[blue]{(1)} ;
      \edge[blue] ; [.\node[blue]{(2,3)}; 
       \edge[blue] ; [.{} ]
       \edge[blue] ; [.\node(perm1)[red]{(2)}; 
	\edge[red] ; [.\node(perm111)[white]{(3)}; ]
	\edge[red] ; [.\node(perm11)[red]{(3)}; ]
       ]
      ]
      \edge[blue] ; [.\node(perm2)[red]{(2)}; 
       \edge[red] ; [.\node(perm21)[red]{(3)}; ]
       \edge[red] ; [.\node(perm22)[red]{(3)};
        \edge[white] ; [.\node[white]{(3)}; ]
        \edge[white] ; [.\node(perm32)[white]{(3)}; ]
       ]
      ]
     ]
\draw[semithick,<->,>=latex] (perm11)..controls +(south west:1) and +(west:1)..(perm1);
\draw[semithick,<->,>=latex] (perm21)..controls +(west:0.7) and +(west:0.8)..(perm2);
\draw[semithick,<->,>=latex] (perm22)..controls +(east:0.7) and +(east:0.8)..(perm2);
\draw[semithick,-,decorate,decoration={brace,amplitude=10pt,mirror}] (perm111.south west) -- node(perm31)[below=10pt] {$2$} (perm11.south east);
\draw[semithick,-,decorate,decoration={brace,amplitude=10pt,mirror}] (perm21.south west) -- node[below=10pt] {$2$} (perm22.south east);
\draw[semithick,-,decorate,decoration={brace,amplitude=10pt,mirror}] (perm31.south west) -- node[below=10pt] {$2^2$} (perm32.south east);
\end{tikzpicture}}\\
\scalebox{0.75}{
\begin{tikzpicture}[level 1/.style={sibling distance=5mm}]
\Tree[.\node[blue]{(1)} ;
      \edge[blue] ; [.\node(perm2)[red]{(2)} ; 
       \edge[red] ; [.\node(perm21)[red]{(3)} ;
        \edge[white] ; [.\node(perm32)[white]{(3)}; ]
        \edge[white] ; [.\node[white]{(3)}; ]
       ] 
       \edge[red] ; [.\node(perm22)[red]{(3)} ; ]
      ]
      \edge[blue] ; [.\node[blue]{(2,3)}; 
       \edge[blue] ; [.{} ]
       \edge[blue] ; [.\node(perm1)[red]{(2)}; 
	\edge[red] ; [.\node(perm12)[white]{(3)}; ]
	\edge[red] ; [.\node(perm11)[red]{(3)}; ]
       ]
      ]
     ]
\draw[semithick,<->,>=latex] (perm11)..controls +(east:1) and +(east:1)..(perm1);
\draw[semithick,<->,>=latex] (perm21)..controls +(west:0.7) and +(west:0.8)..(perm2);
\draw[semithick,<->,>=latex] (perm22)..controls +(east:0.7) and +(east:0.8)..(perm2);
\draw[semithick,-,decorate,decoration={brace,amplitude=10pt,mirror}] (perm21.south west) -- node(perm30)[below=10pt] {$2$} (perm22.south east);
\draw[semithick,-,decorate,decoration={brace,amplitude=10pt,mirror}] (perm12.south west) -- node(perm31)[below=10pt] {$2$} (perm11.south east);
\draw[semithick,-,decorate,decoration={brace,amplitude=10pt,mirror}] (perm32.south west) -- node[below=10pt] {$2^2$} (perm31.south east);
\end{tikzpicture}}\qquad 
\scalebox{0.75}{
\begin{tikzpicture}
\Tree[.\node[blue]{(1)} ;
      \edge[blue] ; [.\node[blue]{(2,3)}; 
       \edge[blue] ; [.{} ]
       \edge[blue] ; [.\node(perm11)[red]{(2)}; 
	\edge[red] ; [.\node(perm13)[white]{(3)}; ]
	\edge[red] ; [.\node(perm12)[red]{(3)}; ]
       ]
      ]
      \edge[blue] ; [.\node[blue]{(2,3)}; 
       \edge[blue] ; [.{} ]
       \edge[blue] ; [.\node(perm21)[red]{(2)}; 
	\edge[red] ; [.\node(perm23)[white]{(3)}; ]
	\edge[red] ; [.\node(perm22)[red]{(3)}; ]
       ]
      ]
     ]
\draw[semithick,<->,>=latex] (perm11)..controls +(east:1) and +(east:1)..(perm12);
\draw[semithick,<->,>=latex] (perm21)..controls +(east:1) and +(east:1)..(perm22);
\draw[semithick,-,decorate,decoration={brace,amplitude=10pt,mirror}] (perm13.south west) -- node(perm31)[below=10pt] {$2$} (perm12.south east);
\draw[semithick,-,decorate,decoration={brace,amplitude=10pt,mirror}] (perm23.south west) -- node(perm32)[below=10pt] {$2$} (perm22.south east);
\draw[semithick,-,decorate,decoration={brace,amplitude=10pt,mirror}] (perm31.south west) -- node[below=10pt] {$2^2$} (perm32.south east);
\end{tikzpicture}}
\caption{Trees representation with a grouping by one element on the root. For a
fixed element, we have $2^2$ possible permutations. Since we have $4$ patterns, we get $2^2 \times 4$ possible permutations for one grouping. Hence,
we finally have $2^2 \times 4 \times 3$ for all possible groupings by one element.}\label{fig:n3g1}
\end{figure}

\iflongversion
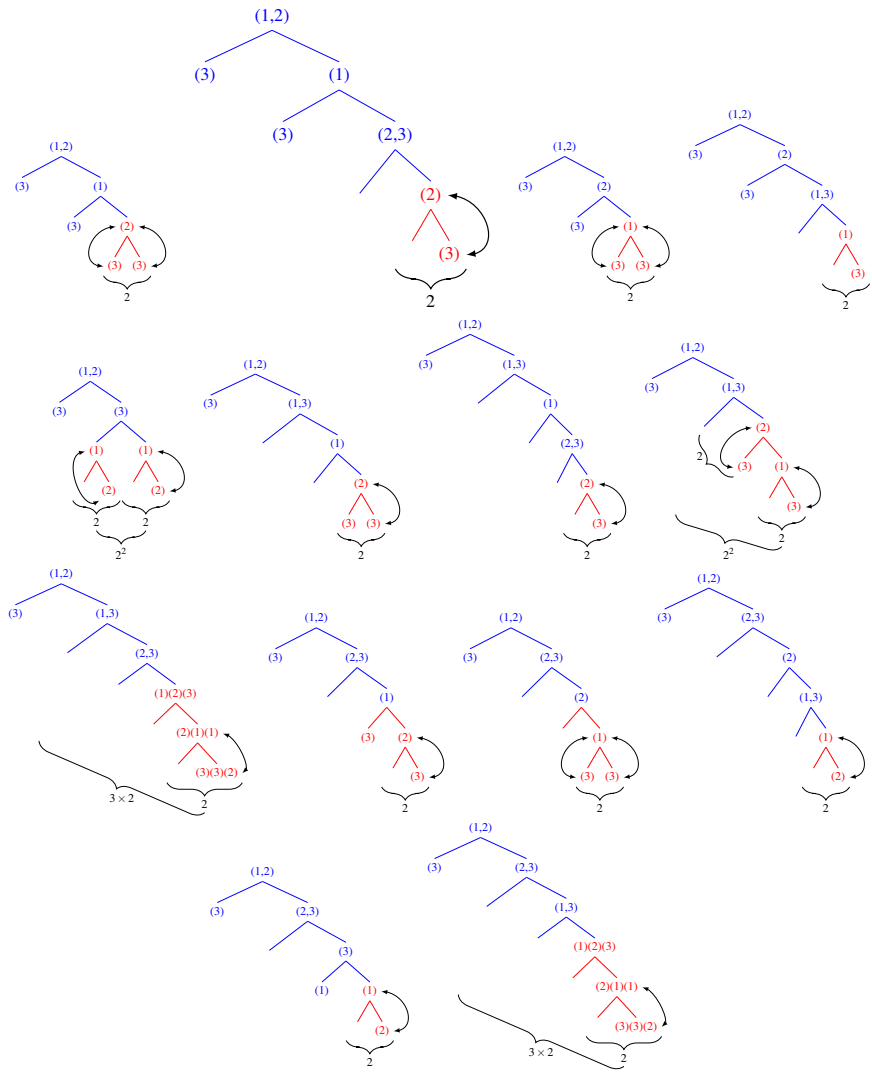
\begin{figure}
\centering 
\scalebox{0.5}{
\begin{tikzpicture}
\Tree[.\node[blue]{(1,2)} ; 
      \edge[blue] ; [.\node[blue]{(3)} ; ] 
      \edge[blue] ; [.\node[blue](perm51){(1)} ;
       \edge[blue]; [.\node[blue](perm2){(3)} ;
	       \edge[white]; [.\node[white]{45} ;
		       \edge[white]; [.\node(perm4)[white]{47} ; ]
		       \edge[white]; [.\node[white]{48} ; ]
	       ]
  	       \edge[white]; [.\node[white]{46} ; ]
       ] 
       \edge[blue]; [.\node[red](perm1){(2)} ; 
         \edge[red] ; [.\node(perm11)[red]{(3)} ; ] 
         \edge[red] ; [.\node(perm12)[red]{(3)} ; ]
        ]
       ]
      ]
     ]
\draw[semithick,<->,>=latex] (perm11)..controls +(west:1) and +(west:1)..(perm1);
\draw[semithick,<->,>=latex] (perm12)..controls +(east:1) and +(east:1)..(perm1);
\draw[semithick,-,decorate,decoration={brace,amplitude=10pt,mirror}] (perm11.south west) -- node(perm31)[below=10pt] {$2$} (perm12.south east);
\end{tikzpicture}}
\scalebox{0.75}{
\begin{tikzpicture}
\Tree[.\node[blue]{(1,2)} ; 
      \edge[blue] ; [.\node[blue]{(3)} ; ] 
      \edge[blue] ; [.\node[blue](perm52){(1)} ;
       \edge[blue]; [.\node[blue](perm71){(3)} ;
	       \edge[white]; [.\node[white]{45} ;
		       \edge[white]; [.\node(perm8)[white]{47} ; ]
		       \edge[white]; [.\node[white]{48} ; ]
	       ]
  	       \edge[white]; [.\node[white]{46} ; ]
       ] 
       \edge[blue];  [.\node[blue]{(2,3)} ;
        \edge[blue]; [.\node[blue]{} ;
	        \edge[white]; [.\node[white]{45} ; ]
	        \edge[white]; [.\node(perm72)[white]{46} ; ]
        ]
        \edge[blue]; [.\node[red](perm61){(2)} ; 
         \edge[red] ; [.\node[white](perm62){(3)} ; ] 
         \edge[red] ; [.\node[red](perm63){(3)} ; ]
        ]
       ]
      ]
     ]
\draw[semithick,<->,>=latex] (perm61)..controls +(east:1) and +(east:1)..(perm63);
\draw[semithick,-,decorate,decoration={brace,amplitude=10pt,mirror}] (perm62.south west) -- node(perm81)[below=10pt] {$2$} (perm63.south east);
\end{tikzpicture}
}\scalebox{0.5}{
\begin{tikzpicture}
\Tree[.\node[blue]{(1,2)} ; 
      \edge[blue] ; [.\node[blue]{(3)} ; ] 
      \edge[blue] ; [.\node[blue]{(2)} ;
       \edge[blue]; [.\node[blue](perm2){(3)} ;
	       \edge[white]; [.\node[white]{45} ;
		       \edge[white]; [.\node(perm4)[white]{47} ; ]
		       \edge[white]; [.\node[white]{48} ; ]
	       ]
  	       \edge[white]; [.\node[white]{46} ; ]
       ] 
        \edge[blue]; [.\node[red](perm1){(1)} ; 
         \edge[red] ; [.\node(perm11)[red]{(3)} ; ] 
         \edge[red] ; [.\node(perm12)[red]{(3)} ; ]
        ]
       ]
      ]
     ]
\draw[semithick,<->,>=latex] (perm11)..controls +(west:1) and +(west:1)..(perm1);
\draw[semithick,<->,>=latex] (perm12)..controls +(east:1) and +(east:1)..(perm1);
\draw[semithick,-,decorate,decoration={brace,amplitude=10pt,mirror}] (perm11.south west) -- node(perm31)[below=10pt] {$2$} (perm12.south east);
\end{tikzpicture}}\scalebox{0.5}{
\begin{tikzpicture}
\Tree[.\node[blue]{(1,2)} ; 
      \edge[blue] ; [.\node[blue]{(3)} ; ] 
      \edge[blue] ; [.\node[blue]{(2)} ;
       \edge[blue]; [.\node[blue](perm2){(3)} ;
	       \edge[white]; [.\node[white]{45} ;
		       \edge[white]; [.\node(perm4)[white]{47} ; ]
		       \edge[white]; [.\node[white]{48} ; ]
	       ]
  	       \edge[white]; [.\node[white]{46} ; ]
       ] 
       \edge[blue];  [.\node[blue]{(1,3)} ;
        \edge[blue]; [.\node[blue]{} ;
	        \edge[white]; [.\node(perm22)[white]{45} ; ]
	        \edge[white]; [.\node(perm21)[white]{46} ; ]
        ] 
         \edge[blue] ; [.\node[red](perm20){(1)} ;
          \edge[red] ; [.\node[white](perm41){(3)} ; ]
          \edge[red] ; [.\node[red](perm42){(3)} ; ]
         ]
       ]
      ]
     ]
\draw[semithick,-,decorate,decoration={brace,amplitude=10pt,mirror}] (perm41.south west) -- node(perm31)[below=10pt] {$2$} (perm42.south east);
\end{tikzpicture}}\\
\scalebox{0.5}{
\begin{tikzpicture}
\Tree[.\node[blue]{(1,2)} ; 
      \edge[blue] ; [.\node[blue]{(3)} ; ] 
      \edge[blue] ; [.\node[blue]{(3)} ;
       \edge[blue]; [.\node(perm11)[red]{(1)} ;
	       \edge[red]; [.\node(perm10)[white]{(2)} ; ]
           \edge[red]; [.\node(perm12)[red]{(2)} ; ]
       ] 
       \edge[blue];  [.\node(perm21)[red]{(1)} ;
        \edge[red]; [.\node(perm20)[white]{(2)} ; ]
        \edge[red]; [.\node(perm22)[red]{(2)} ; ]
       ]
      ]
     ]
\draw[semithick,<->,>=latex] (perm11)..controls +(west:1) and +(south west:1)..(perm12);
\draw[semithick,<->,>=latex] (perm21)..controls +(east:1) and +(east:1)..(perm22);
\draw[semithick,-,decorate,decoration={brace,amplitude=10pt,mirror}] (perm10.south west) -- node(perm90)[below=10pt] {$2$} (perm12.south east);
\draw[semithick,-,decorate,decoration={brace,amplitude=10pt,mirror}] (perm20.south west) -- node(perm91)[below=10pt] {$2$} (perm22.south east);
\draw[semithick,-,decorate,decoration={brace,amplitude=10pt,mirror}] (perm90.south) -- node[below=10pt] {$2^2$} (perm91.south);
\end{tikzpicture}}\scalebox{0.5}{
\begin{tikzpicture}
\Tree[.\node[blue]{(1,2)} ; 
      \edge[blue] ; [.\node(perm2)[blue]{(3)} ; ] 
      \edge[blue] ; [.\node[blue]{(1,3)} ;
       \edge[blue]; [.\node[blue]{} ;
	       \edge[white]; [.\node[white]{45} ;
		       \edge[white]; [.\node(perm4)[white]{47} ; ]
		       \edge[white]; [.\node[white]{48} ; ]
	       ]
  	       \edge[white]; [.\node[white]{46} ; ]
       ] 
       \edge[blue];  [.\node[blue]{(1)} ;
        \edge[blue]; [.\node(perm6)[blue]{} ;
	        \edge[white]; [.\node(perm22)[white]{45} ; ]
	        \edge[white]; [.\node(perm21)[white]{46} ; ]
        ]
            \edge[blue]; [.\node[red](perm30){(2)} ;
            	\edge[red]; [.\node[red](perm31){(3)} ; ]
            	\edge[red]; [.\node[red](perm32){(3)} ; ]
            ]
       ]
      ]
     ]
\draw[semithick,<->,>=latex] (perm30)..controls +(east:1) and +(east:1)..(perm32);
\draw[semithick,-,decorate,decoration={brace,amplitude=10pt,mirror}] (perm31.south west) -- node(perm9)[below=10pt] {$2$} (perm32.south east);
\end{tikzpicture}}\scalebox{0.5}{
\begin{tikzpicture}
\Tree[.\node[blue]{(1,2)} ; 
      \edge[blue] ; [.\node[blue]{(3)} ; ] 
      \edge[blue] ; [.\node[blue]{(1,3)} ;
       \edge[blue]; [.\node[blue]{} ;
	       \edge[white]; [.\node[white]{45} ;
		       \edge[white]; [.\node(perm4)[white]{47} ; ]
		       \edge[white]; [.\node[white]{48} ; ]
	       ]
  	       \edge[white]; [.\node[white]{46} ; ]
       ] 
       \edge[blue];  [.\node[blue]{(1)} ;
        \edge[blue]; [.\node[blue](perm2){} ;
	        \edge[white]; [.\node(perm22)[white]{45} ; ]
	        \edge[white]; [.\node(perm21)[white]{46} ; ]
        ]
        \edge[blue]; [.\node[blue](perm12){(2,3)} ;
        	\edge[blue]; [.\node(perm6)[red]{} ; ]
            \edge[blue]; [.\node[red](perm30){(2)} ;
            	\edge[red]; [.\node[white](perm31){(3)} ; ]
            	\edge[red]; [.\node[red](perm32){(3)} ; ]
            ]
        ]
       ]
      ]
     ]
\draw[semithick,<->,>=latex] (perm30)..controls +(east:1) and +(east:1)..(perm32);
\draw[semithick,-,decorate,decoration={brace,amplitude=10pt,mirror}] (perm31.south west) -- node(perm9)[below=10pt] {$2$} (perm32.south east);
\end{tikzpicture}}\scalebox{0.5}{
\begin{tikzpicture}
\Tree[.\node[blue]{(1,2)} ; 
      \edge[blue] ; [.\node[blue]{(3)} ; ] 
      \edge[blue] ; [.\node[blue]{(1,3)} ;
       \edge[blue]; [.\node[blue](perm2){} ;
	       \edge[white]; [.\node[white]{45} ;
		       \edge[white]; [.\node(perm4)[white]{47} ; ]
		       \edge[white]; [.\node[white]{48} ; ]
	       ]
  	       \edge[white]; [.\node[white]{46} ; ]
       ] 
        \edge[blue]; [.\node[red](perm12){(2)} ;
        	\edge[red] ; [.\node(perm6)[red]{(3)} ; ]
            \edge[red] ; [.\node[red](perm30){(1)} ;
            	\edge[red]; [.\node[white](perm31){(2)} ; ]
            	\edge[red]; [.\node[red](perm32){(3)} ; ]
            ]
        ]
       ]
      ]
     ]
\draw[semithick,<->,>=latex] (perm30)..controls +(east:1) and +(east:1)..(perm32);
\draw[semithick,<->,>=latex] (perm6)..controls +(west:1) and +(west:1)..(perm12);
\draw[semithick,-,decorate,decoration={brace,amplitude=10pt,mirror}] (perm2.south west) -- node(perm10)[left=7pt] {$2$} (perm6.south west);
\draw[semithick,-,decorate,decoration={brace,amplitude=10pt,mirror}] (perm31.south west) -- node(perm9)[below=10pt] {$2$} (perm32.south east);
\draw[semithick,-,decorate,decoration={brace,amplitude=10pt,mirror}] (perm4.south) -- node[below=10pt] {$2^2$} (perm9.south);
\end{tikzpicture}}\\
\scalebox{0.5}{
\begin{tikzpicture}
\Tree[.\node[blue]{(1,2)} ; 
      \edge[blue] ; [.\node[blue]{(3)} ; ] 
      \edge[blue] ; [.\node[blue]{(1,3)} ;
       \edge[blue]; [.\node[blue](perm2){} ;
	       \edge[white]; [.\node[white]{45} ;
		       \edge[white]; [.\node(perm4)[white]{47} ; ]
		       \edge[white]; [.\node[white]{48} ; ]
	       ]
  	       \edge[white]; [.\node[white]{46} ; ]
       ] 
       \edge[blue];  [.\node[blue]{(2,3)} ;
        \edge[blue]; [.\node[white]{(1)} ;
	        \edge[white]; [.\node(perm22)[white]{45} ; ]
	        \edge[white]; [.\node(perm21)[white]{46} ; ]
        ]
        \edge[blue]; [.\node[red](perm12){(1)(2)(3)} ;
        	\edge[red] ; [.\node(perm6)[white]{(1)} ; ]
            \edge[red] ; [.\node[red](perm30){(2)(1)(1)} ;
            	\edge[red]; [.\node[white](perm31){(3)} ; ]
            	\edge[red]; [.\node[red](perm32){(3)(3)(2)} ; ]
            ]
        ]
       ]
      ]
     ]
\draw[semithick,<->,>=latex] (perm30)..controls +(east:1) and +(east:1)..(perm32);
\draw[semithick,-,decorate,decoration={brace,amplitude=10pt,mirror}] (perm31.south west) -- node(perm9)[below=10pt] {$2$} (perm32.south east);
\draw[semithick,-,decorate,decoration={brace,amplitude=10pt,mirror}] (perm4.south) -- node[below=10pt] {$3 \times 2$} (perm9.south);
\end{tikzpicture}}
\scalebox{0.5}{
\begin{tikzpicture}
\Tree[.\node[blue]{(1,2)} ; 
      \edge[blue] ; [.\node[blue]{(3)} ; ] 
      \edge[blue] ; [.\node[blue]{(2,3)} ;
       \edge[blue]; [.\node[blue](perm2){} ;
	       \edge[white]; [.\node[white]{45} ;
		       \edge[white]; [.\node(perm4)[white]{47} ; ]
		       \edge[white]; [.\node[white]{48} ; ]
	       ]
  	       \edge[white]; [.\node[white]{46} ; ]
       ] 
        \edge[blue]; [.\node[blue](perm12){(1)} ;
        	\edge[red] ; [.\node(perm6)[red]{(3)} ; ]
            \edge[red] ; [.\node[red](perm30){(2)} ;
            	\edge[red]; [.\node[white](perm31){(2)} ; ]
            	\edge[red]; [.\node[red](perm32){(3)} ; ]
            ]
        ]
       ]
      ]
     ]
\draw[semithick,<->,>=latex] (perm30)..controls +(east:1) and +(east:1)..(perm32);
\draw[semithick,-,decorate,decoration={brace,amplitude=10pt,mirror}] (perm31.south west) -- node(perm9)[below=10pt] {$2$} (perm32.south east);
\end{tikzpicture}}\scalebox{0.5}{
\begin{tikzpicture}
\Tree[.\node[blue]{(1,2)} ; 
      \edge[blue] ; [.\node[blue]{(3)} ; ] 
      \edge[blue] ; [.\node[blue]{(2,3)} ;
       \edge[blue]; [.\node[blue](perm2){} ;
	       \edge[white]; [.\node[white]{45} ;
		       \edge[white]; [.\node(perm4)[white]{47} ; ]
		       \edge[white]; [.\node[white]{48} ; ]
	       ]
  	       \edge[white]; [.\node[white]{46} ; ]
       ] 
        \edge[blue]; [.\node[blue](perm12){(2)} ;
        	\edge[red] ; [.\node(perm6)[white]{(3)} ; ]
            \edge[red] ; [.\node[red](perm30){(1)} ;
            	\edge[red]; [.\node[red](perm31){(3)} ; ]
            	\edge[red]; [.\node[red](perm32){(3)} ; ]
            ]
        ]
       ]
      ]
     ]
\draw[semithick,<->,>=latex] (perm30)..controls +(west:1) and +(west:1)..(perm31);
\draw[semithick,<->,>=latex] (perm30)..controls +(east:1) and +(east:1)..(perm32);
\draw[semithick,-,decorate,decoration={brace,amplitude=10pt,mirror}] (perm31.south west) -- node(perm9)[below=10pt] {$2$} (perm32.south east);
\end{tikzpicture}}\scalebox{0.5}{
\begin{tikzpicture}
\Tree[.\node[blue]{(1,2)} ; 
      \edge[blue] ; [.\node[blue]{(3)} ; ] 
      \edge[blue] ; [.\node[blue]{(2,3)} ;
       \edge[blue]; [.\node[blue]{} ;
	       \edge[white]; [.\node[white]{45} ;
		       \edge[white]; [.\node(perm4)[white]{47} ; ]
		       \edge[white]; [.\node[white]{48} ; ]
	       ]
  	       \edge[white]; [.\node[white]{46} ; ]
       ] 
       \edge[blue];  [.\node[blue]{(2)} ;
        \edge[blue]; [.\node[blue](perm2){} ;
	        \edge[white]; [.\node(perm22)[white]{45} ; ]
	        \edge[white]; [.\node(perm21)[white]{46} ; ]
        ]
        \edge[blue]; [.\node[blue](perm12){(1,3)} ;
        	\edge[blue]; [.\node(perm6)[red]{} ; ]
            \edge[blue]; [.\node[red](perm30){(1)} ;
            	\edge[red]; [.\node[white](perm31){(2)} ; ]
            	\edge[red]; [.\node[red](perm32){(2)} ; ]
            ]
        ]
       ]
      ]
     ]
\draw[semithick,<->,>=latex] (perm30)..controls +(east:1) and +(east:1)..(perm32);
\draw[semithick,-,decorate,decoration={brace,amplitude=10pt,mirror}] (perm31.south west) -- node(perm9)[below=10pt] {$2$} (perm32.south east);
\end{tikzpicture}}\\
\scalebox{0.5}{
\begin{tikzpicture}
\Tree[.\node[blue]{(1,2)} ; 
      \edge[blue] ; [.\node(perm2)[blue]{(3)} ; ] 
      \edge[blue] ; [.\node[blue]{(2,3)} ;
       \edge[blue]; [.\node[blue]{} ;
	       \edge[white]; [.\node[white]{45} ;
		       \edge[white]; [.\node(perm4)[white]{47} ; ]
		       \edge[white]; [.\node[white]{48} ; ]
	       ]
  	       \edge[white]; [.\node[white]{46} ; ]
       ] 
       \edge[blue];  [.\node[blue]{(3)} ;
        \edge[blue]; [.\node(perm6)[blue]{(1)} ;
	        \edge[white]; [.\node(perm22)[white]{(1)} ; ]
	        \edge[white]; [.\node(perm21)[white]{46} ; ]
        ]
            \edge[blue]; [.\node[red](perm30){(1)} ;
            	\edge[red]; [.\node[white](perm31){(2)} ; ]
            	\edge[red]; [.\node[red](perm32){(2)} ; ]
            ]
       ]
      ]
     ]
\draw[semithick,<->,>=latex] (perm30)..controls +(east:1) and +(east:1)..(perm32);
\draw[semithick,-,decorate,decoration={brace,amplitude=10pt,mirror}] (perm31.south west) -- node(perm9)[below=10pt] {$2$} (perm32.south east);
\end{tikzpicture}}\scalebox{0.5}{
\begin{tikzpicture}
\Tree[.\node[blue]{(1,2)} ; 
      \edge[blue] ; [.\node[blue]{(3)} ; ] 
      \edge[blue] ; [.\node[blue]{(2,3)} ;
       \edge[blue]; [.\node[blue](perm2){} ;
	       \edge[white]; [.\node[white]{45} ;
		       \edge[white]; [.\node(perm4)[white]{47} ; ]
		       \edge[white]; [.\node[white]{48} ; ]
	       ]
  	       \edge[white]; [.\node[white]{46} ; ]
       ] 
       \edge[blue];  [.\node[blue]{(1,3)} ;
        \edge[blue]; [.\node[white]{(1)} ;
	        \edge[white]; [.\node(perm22)[white]{45} ; ]
	        \edge[white]; [.\node(perm21)[white]{46} ; ]
        ]
        \edge[blue]; [.\node[red](perm12){(1)(2)(3)} ;
        	\edge[red] ; [.\node(perm6)[white]{(1)} ; ]
            \edge[red] ; [.\node[red](perm30){(2)(1)(1)} ;
            	\edge[red]; [.\node[white](perm31){(3)} ; ]
            	\edge[red]; [.\node[red](perm32){(3)(3)(2)} ; ]
            ]
        ]
       ]
      ]
     ]
\draw[semithick,<->,>=latex] (perm30)..controls +(east:1) and +(east:1)..(perm32);
\draw[semithick,-,decorate,decoration={brace,amplitude=10pt,mirror}] (perm31.south west) -- node(perm9)[below=10pt] {$2$} (perm32.south east);
\draw[semithick,-,decorate,decoration={brace,amplitude=10pt,mirror}] (perm4.south) -- node[below=10pt] {$3 \times 2$} (perm9.south);
\end{tikzpicture}}
\caption{Tree representations with a grouping by two elements on the root. For 10 fixed elements, we have $2$ possible permutations, for 2 fixed elements, we have $2$ possible permutations, and for $2$ possible permutations, we have $6$ possible permutations. Hence,
we finally have $2 \times 10 + 4 \times 2 + 6 \times 2$ for all possible groupings by two elements.}\label{fig:n3g2}
\end{figure}

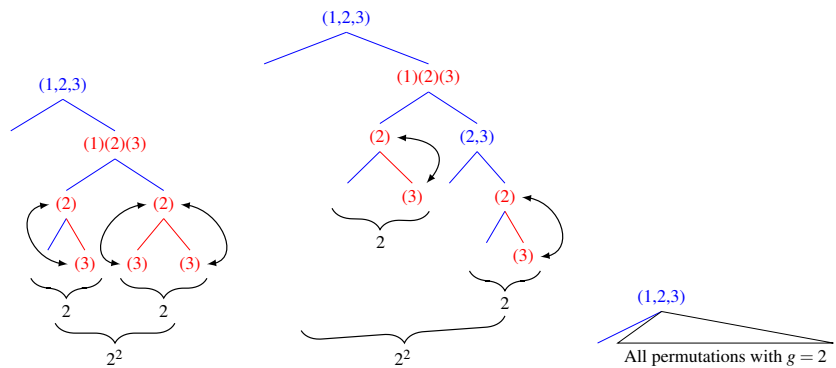
\begin{figure}[!ht]
\centering 
\scalebox{0.75}{%
\begin{tikzpicture}
\Tree[.\node[blue]{(1,2,3)} ; 
      \edge[blue] ; [.\node[white]{(3)} ; ] 
      \edge[blue] ; [.\node[red]{(1)(2)(3)} ;
       \edge[blue]; [.\node[red](perm10){(2)} ;
	       \edge[blue]; [.\node[white](perm12){(5)} ; ]
  	       \edge[red]; [.\node[red](perm11){(3)}  ; ]
       ] 
       \edge[blue];  [.\node[red](perm30){(2)} ;
        \edge[red]; [.\node[red](perm31){(3)} ;
	        \edge[white]; [.\node(perm22)[white]{45} ; ]
	        \edge[white]; [.\node(perm21)[white]{46} ; ]
        ]
        \edge[red]; [.\node[red](perm32){(3)} ; ]
       ]
      ]
     ]
\draw[semithick,<->,>=latex] (perm30)..controls +(east:1) and +(east:1)..(perm32);
\draw[semithick,<->,>=latex] (perm30)..controls +(west:1) and +(west:1)..(perm31);
\draw[semithick,<->,>=latex] (perm10)..controls +(west:1) and +(west:1)..(perm11);
\draw[semithick,-,decorate,decoration={brace,amplitude=10pt,mirror}] (perm31.south west) -- node(perm90)[below=10pt] {$2$} (perm32.south east);
\draw[semithick,-,decorate,decoration={brace,amplitude=10pt,mirror}] (perm12.south west) -- node(perm91)[below=10pt] {$2$} (perm11.south east);
\draw[semithick,-,decorate,decoration={brace,amplitude=10pt,mirror}] (perm91.south west) -- node[below=10pt] {$2^2$} (perm90.south east);
\end{tikzpicture}}\scalebox{0.75}{%
\begin{tikzpicture}
\Tree[.\node[blue]{(1,2,3)} ; 
      \edge[blue] ; [.\node(perm2)[white]{(3)} ; ] 
      \edge[blue] ; [.\node[red]{(1)(2)(3)} ;
       \edge[blue]; [.\node(perm10)[red]{(2)} ;
	       \edge[blue]; [.\node(perm110)[white]{45} ;
           	\edge[white]; [.\node[white]{(3)} ;
            	\edge[white]; [.\node(perm50)[white]{(3)}; ]
                \edge[white]; [.\node[white]{(3)}; ]
            ]
            \edge[white]; [.\node[white]{(3)} ; ]
           ]
  	       \edge[red]; [.\node(perm11)[red]{(3)} ; ]
       ] 
       \edge[blue];  [.\node[blue]{(2,3)} ;
        \edge[blue]; [.\node(perm6)[white]{(3)} ; ]
        \edge[blue]; [.\node[red](perm12){(2)} ; 
         \edge[blue]; [.\node[white](perm13){(3)} ; ]
         \edge[red]; [.\node[red](perm14){(3)} ; ]
        ]
       ]
      ]
     ]
\draw[semithick,<->,>=latex] (perm10)..controls +(east:1) and +(north east:1)..(perm11);
\draw[semithick,<->,>=latex] (perm12)..controls +(east:1) and +(east:1)..(perm14);
\draw[semithick,-,decorate,decoration={brace,amplitude=10pt,mirror}] (perm110.south west) -- node(perm90)[below=10pt] {$2$} (perm11.south east);
\draw[semithick,-,decorate,decoration={brace,amplitude=10pt,mirror}] (perm13.south west) -- node(perm91)[below=10pt] {$2$} (perm14.south east);
\draw[semithick,-,decorate,decoration={brace,amplitude=10pt,mirror}] (perm50.south) -- node[below=10pt] {$2^2$} (perm91.south);
\end{tikzpicture}}\scalebox{0.75}{%
\begin{tikzpicture}
\Tree[.\node[blue]{(1,2,3)} ; 
      \edge[blue] ; [.\node(perm2)[white]{(3)} ; ] 
      \edge[roof] ; [.\node[black]{All permutations with $g=2$} ; ]
     ]
\end{tikzpicture}}
\caption{Trees representation with a grouping by three elements on the root. For a fixed element at the upper left corner side, we
have $2^2$ possible permutations. For the upper right corner side, we get $2^2$. We replace the subroot of the fixed trees and get $(2^2 + 2^2) \times 3$. We also have the $40 \times 3$ trees from the grouping of two ($g=2$). Hence, we have $40 \times 3 + (2^2 + 2^2) \times 3$}\label{fig:n3g3}
\end{figure}
\fi

\end{document}